\documentclass[10pt, conference, letterpaper]{ IEEEtran}

\usepackage{amsmath,algorithm, algorithmicx, algpseudocode,bbm}
\usepackage{textcomp}
\usepackage{amsfonts,xcolor} 
\usepackage{amsthm}

\usepackage[normalem]{ulem}
\usepackage{subfigure}
\usepackage{graphicx}
\usepackage{cite}
\usepackage{url}
\usepackage{bbm}
\usepackage{multirow}
\usepackage{amssymb}
\usepackage{titlesec}
\usepackage{enumitem}
\usepackage{booktabs}

\newtheorem{theorem}{Theorem}

\newtheorem{corollary}[theorem]{Corollary}

\newtheorem{lemma}[theorem]{Lemma}

\newcommand{\haoran}[1]{{\color{magenta}#1}}

\definecolor{blue}{rgb}{0,0,0}
\definecolor{magenta}{rgb}{0,0,0}
\definecolor{violet}{rgb}{0,0,0}
\definecolor{brown}{rgb}{0,0,0}
\definecolor{purple}{rgb}{0,0,0}
\definecolor{teal}{rgb}{0,0,0}
\definecolor{red}{rgb}{0,0,0}
\definecolor{green}{rgb}{0,0,0}
\definecolor{orange}{rgb}{0,0,0}
\definecolor{pink}{rgb}{0,0,0}

\ifCLASSINFOpdf
\else
\fi
\begin{document}

%
\title{Fair Concurrent Training of Multiple Models in Federated Learning}

\author{
\IEEEauthorblockN{Marie~Siew\IEEEauthorrefmark{1},
Haoran~Zhang\IEEEauthorrefmark{2}, Jong-Ik Park\IEEEauthorrefmark{2}, Yuezhou Liu\IEEEauthorrefmark{3}, Yichen Ruan, Lili Su\IEEEauthorrefmark{3}, Stratis Ioannidis\IEEEauthorrefmark{3},\\
Edmund Yeh\IEEEauthorrefmark{3}, and
Carlee~Joe-Wong\IEEEauthorrefmark{2}
}
\IEEEauthorblockA{\IEEEauthorrefmark{1}Information Systems Technology and Design Pillar, Singapore University of Technology and Design, 487372 Singapore}
\IEEEauthorblockA{\IEEEauthorrefmark{2}Electrical and Computer Engineering,
Carnegie Mellon University, Pittsburgh, PA 15213 USA}
\IEEEauthorblockA{\IEEEauthorrefmark{3}Electrical and Computer Engineering,
Northeastern University, Boston, MA 02115 USA}
\IEEEauthorblockA{Emails:  marie\_siew@sutd.edu.sg, \{haoranz5, jongikp, cjoewong\}@andrew.cmu.edu, \{liu.yuez, l.su\}@northeastern.edu,\\
yichenr@alumni.cmu.edu, stratis.ioannidis@gmail.com, eyeh@ece.neu.edu
}
}

\maketitle

\begin{abstract}  
    Federated learning (FL) enables collaborative learning across multiple clients. In most FL work, all clients train a single learning task. \textcolor{purple}{However, the recent proliferation of FL applications may increasingly require multiple FL tasks to \textcolor{brown}{be trained simultaneously,} sharing clients' computing resources,}
    \textcolor{brown}{which we call Multiple-Model Federated Learning (MMFL).}
    \textcolor{brown}{Current MMFL algorithms use na\"ive average-based client-task allocation schemes} that
    \textcolor{purple}{
    often lead to unfair performance when FL tasks have heterogeneous difficulty levels, \textcolor{orange}{as the more difficult tasks} 
    may need more client participation to train effectively. 
    Furthermore, in the MMFL setting, 
    we face a further challenge that some clients may prefer training specific tasks to others, and may not even be willing to train other tasks, e.g., due to high computational costs, which may exacerbate unfairness in training outcomes \textcolor{green}{across tasks.}}
    We address both challenges by firstly designing FedFairMMFL, a difficulty-aware algorithm that \textcolor{brown}{dynamically allocates} clients to tasks in each training round, based on the tasks' current performance levels.
    We provide guarantees on the resulting task fairness and FedFairMMFL's convergence rate. We then propose novel auction designs that incentivizes clients to train multiple tasks, so as to fairly distribute clients' training efforts across the tasks, and extend our convergence guarantees to this setting.
    We 
    finally evaluate our algorithm with multiple sets of learning tasks on real world datasets, showing that our algorithm improves fairness by improving the final model accuracy and convergence speed of the worst performing tasks, while maintaining the average accuracy across tasks.
\end{abstract}

%
\IEEEpeerreviewmaketitle

\section{Introduction}
\label{section:Introduction}
Federated learning (FL) enables multiple devices or users (called clients) to collectively train a machine learning model ~\cite{mcmahan2017communication, yao2023fedrule,aivodji2019iotfla, liu2023cache,liao2023accelerating}, \textcolor{violet}{with data staying at the user side.} 
Several recent works have demonstrated FL's potential in many applications, such as smart homes~\cite{yao2023fedrule,aivodji2019iotfla}, analyzing healthcare data~\cite{nguyen2022federated}, and 
next word prediction~\cite{imteaj2019distributed}. 
Despite the expanding range of FL applications, however, the vast majority of FL research assumes that clients jointly train just \textit{one} model during FL \cite{liao2023adaptive,liao2024mergesfl,liao2024parallelsfl}. 
In practice, 
clients may increasingly need to train multiple tasks (models) at the same time, 
e.g., smartphones contributing to training both next-word prediction and advertisement recommendation models; 
\textcolor{orange}{data from wearables contributing to the FL training of various human physical activity recognition and health monitoring (e.g. heart rate prediction or early disease detection) models concurrently;}
hospitals contributing to training separate prediction models for heart disease and breast cancer;
\textcolor{magenta}{or intelligent vehicles generating data for driving-related tasks (e.g. control and localization), and
safety-related tasks (e.g., obstacle detection and driver monitoring).}
\textcolor{purple}{Clients could also train multiple versions of a model, in order to select the best-performing one for deployment, e.g., training many different neural network architectures on the same data.}
While a na\"ive solution would simply train all models sequentially, doing so is inefficient, as the required training time would scale linearly in the number of models \cite{askin2024fedast}.
\textcolor{magenta}{This approach also makes the later models wait for the prior models to finish training, which may be unfair.}
Moreover, training multiple tasks concurrently in one FL setting amortizes the real-world costs of recruiting clients and coordinating the training time across clients, improving resource efficiency.
We call this setting \textcolor{magenta}{in which FL clients concurrently train multiple tasks, \textit{Multiple-Model Federated Learning} (MMFL) \cite{bhuyan2022multiProofs,bhuyan2022multi, li2021efficient,liu2022multi,chang2023asynchronous,siew2023fair,askin2024fedast}, and some prior research has demonstrated that this concurrent training approach can accelerate the overall training process}
\textcolor{violet}{\cite{bhuyan2022multiProofs,bhuyan2022multi}.} 

\textcolor{magenta}{Prior works on MMFL \cite{bhuyan2022multi, bhuyan2022multiProofs, liu2022multi, li2021efficient, chang2023asynchronous} seek to optimize the client-task allocation each training round, given the resource constraints of devices. 
These works allocate clients to tasks so as to 
optimize the \textit{average} task accuracy or training time,} and often assume that training tasks have \textit{similar difficulties}, e.g., training multiple copies of the same model~\cite{bhuyan2022multi,bhuyan2022multiProofs}.

\textcolor{violet}{In practice, however, FL tasks \textcolor{magenta}{may have varying difficulty levels, and} may be requested by different entities and users. Therefore, it is important to ensure comparable convergence time and converged accuracies across tasks of \textit{varied difficulty levels}.} 
\textcolor{magenta}{For instance, wearable devices frequently train lightweight models for monitoring simpler metrics like step counting or heart rate tracking, concurrently with more complex tasks like early disease detection (e.g., arrhythmia detection). 
We wish to concurrently train these different models, without favouring one model over the other.
}
\textcolor{violet}{Likewise, in applications like autonomous driving, improving 
the worst-performing FL task helps to improve the overall quality of experience for users.}
\textcolor{purple}{Indeed, even if the multiple tasks are different models trained by the same FL operator, the operator will not know which model performs best until the training of \textit{all} models is complete: we therefore would want to \textcolor{brown}{dynamically} allocate clients to tasks so as to optimize the performance of the worst-performing task, which is accomplished by optimizing for fair task accuracies, as well as high average accuracy.} 
Existing MMFL works that focus only on average task accuracy cannot also ensure fair performance across tasks.
This goal of \textit{fairness} is widely accepted as an important performance criterion for generic computing jobs~\cite{joe2013multiresource,altman2008generalized,bampis2018fair}, but is less well-studied in MMFL.
To the best of our knowledge, we are the first to extend the MMFL setting to consider \textbf{fair training} of concurrent FL models. 


\subsection{Research Challenges in Fair MMFL}

In considering fairness across multiple different models, we seek to ensure that each model achieves \textbf{comparable training performance}, \textcolor{blue}{in terms of their converged accuracies, or the time taken to reach threshold accuracy levels.} The first challenge to doing so is that \textbf{the multiple models can have varying difficulties} 
and thus require different amounts of data and computing resources to achieve the same accuracy. Existing MMFL training algorithms, e.g., random allocation, or the round-robin algorithm proposed in~\cite{bhuyan2022multiProofs}, allocate clients to tasks uniformly on expectation, \textcolor{violet}{or assume that tasks have similar or the same difficulty, 
e.g., training multiple copies of the same model \cite{bhuyan2022multi}}. 
However, while such strategies arguably achieve fairness, or at least equality, with respect to resource (i.e., client) allocation to different models, they may lead to ``unfair'' outcomes. \textcolor{purple}{For example, smart home devices may train simpler temperature prediction models that take low-dimensional timeseries as input data, as well as complex person identification ones that use high-dimensional image input data~\cite{aivodji2019iotfla}.} 
\textcolor{teal}{
Allocating equal resources to each model may then yield poor performance for the more difficult ones.} 
\textcolor{pink}{
As the easier tasks may need less resources in the form of clients compared to the more difficult tasks, a further optimization of client-task allocation could lead to faster convergence and higher accuracies for the more difficult tasks, while maintaining the performance of the easier tasks.}

Much work in computer systems and networks has focused on 
fairness in allocating resources to different computing jobs or network flows, given their heterogeneous resource requirements \cite{joe2013multiresource,altman2008generalized,bampis2018fair}. 
Directly applying these allocation algorithms to MMFL, however, is difficult: here, each ``resource'' is a client, and the performance (e.g. accuracy) improvement realized by allocating this client to a given FL task is unknown in advance and will change over time. 
This is because it depends on the current model state and the data distribution across clients, both of which are highly stochastic. We must therefore handle \textbf{uncertainty and dynamics} \textcolor{purple}{in how much each client contributes to each task's training performance}.

Finally, purely viewing MMFL fairness as a matter of assigning clients to tasks ignores the fact that \textbf{clients may not be equally willing to contribute resources to each model}. For example, smartphone users may benefit from better next-word prediction models if they text frequently, but may derive limited benefit from a better advertisement recommendation model. Also, users may refuse to compute updates 
if they judge that associated computing costs outweigh personal benefits. If multiple users prefer one model over another, then the disfavored model will overfit to users willing to train it, again leading to an unfair outcome. In an extreme case, no users may be assigned to train this task at all. Prior work\textcolor{blue}{\cite{deng2021fair}} has proposed mechanisms to incentivize clients to help train multiple non-simultaneously trained FL models, 
but \textcolor{blue}{their focus was on optimizing} the task take-up rate, 
instead of ensuring that clients are \textit{fairly distributed} across models.

\subsection{Our Contributions}
Our goal is to design methods to ensure fair concurrent training of MMFL, addressing the challenges above. To do so, we utilize the popular metric of 
$\alpha$-fairness, commonly used for fair resource allocation in networking \cite{lan2010axiomatic} and previously proposed to measure fairness across clients in training a single FL model \cite{li2019fair}.
We use $\alpha$-fairness and Max-Min fairness \textcolor{red}{(a special case of $\alpha$-fairness when $\alpha \rightarrow \infty$ \cite{lan2010axiomatic})} to design novel algorithms that both \textit{incentivize} clients to contribute to FL tasks and \textit{adaptively allocate} these incentivized clients to multiple tasks during training. Intuitively, these algorithms ensure that client effort is spread across the MMFL tasks in accordance with their difficulty, helping to achieve fairness with respect to task performance without compromising average task performance.
After discussing related works in Section~\ref{sec:related}, we make the following contributions:
\begin{itemize}
\item We first suppose that all clients are willing to train all models, and propose \textbf{FedFairMMFL}, an \textbf{algorithm to fairly allocate clients to tasks} (i.e., models) in Section~\ref{sec:formulation}. Our algorithm \emph{dynamically and adaptively} allocates clients to tasks according to the current performance levels of all tasks.
We design \textbf{FedFairMMFL} to approximately optimize the $\alpha$-fairness of all tasks' training accuracies. 
    \item From the perspective of a single model, our client-task allocation algorithm in MMFL can be viewed as a form of \textit{client selection}. 
    With this perspective, we \textbf{analyze the convergence of MMFL tasks} under \textbf{FedFairMMFL} in Section~\ref{sec:fairConvgProofs}, showing that \textbf{FedFairMMFL} preferentially accelerates convergence of more difficult tasks relative to an equal allocation of clients to tasks (Theorem~\ref{theorem:convg_boundofOneItr} and Corollary \textcolor{blue}{\ref{corollary:onAlphaAndBSelTerm}}) and allows each model's training to converge (Corollary \textcolor{blue}{\ref{corollary:firstConvgBound}}).
    \item We next suppose that clients may have unequal interest in training different tasks and may need to be incentivized. 
    We discuss \textit{Budget-Fairness} and \textit{Max-min Fairness}, and we propose a \textbf{max-min fair auction mechanism} 
    to ensure that a fair number of clients are incentivized across tasks. 
    We show that it incentivizes truthful client bids with high probability (Theorem \ref{theorem:casesUntruthful} and Corollary  \ref{corollary:UntrImpactMaxmin}). We also show that it leads to more fair outcomes: it allocates more clients to each model, compared to other auction mechanisms that either ignore fairness or allocate the same budget to each task (Corollary \ref{cor:auctions_compare}).
    Following which, 
    we provide a convergence bound showing how our proposed incentive mechanisms coupled with \textbf{FedFairMMFL} impact the convergence rate (Theorem \ref{theorem:AuctionConvg}).
    \item We finally \textbf{validate} our client incentivization and allocation algorithms on multiple FL datasets. Our results show that our algorithm achieves a higher minimum accuracy across tasks than random allocation, round robin \textcolor{orange}{and client-fairness} baselines, \textcolor{black}{and a lower variance across task accuracy levels,} 
    \textcolor{orange}{while maintaining the average task accuracy}.
    \textcolor{blue}{Our max-min client incentivization further improves the minimum accuracy in the budget-constrained regime.} 
\end{itemize}

We discuss future directions and conclude in Section~\ref{sec:conclusion}.

\section{Related Work}\label{sec:related}

We distinguish our MMFL setting from that of \textbf{multi-task learning} (MTL), in which multiple training tasks have the same or a similar  structure~\cite{multiTaskSurvey}, e.g., a similar underlying lower rank structure, \textcolor{orange}{and are jointly learnt.} 
Unlike clustered FL~\cite{ruan2022fedsoft,hu2022fedcross}, 
our work is different from combining MTL with FL. Instead we consider multiple \textit{distinct} learning tasks, \textcolor{orange}{and their competition for limited client training resources.} 
Our analysis of MMFL performance is closer to prior convergence analysis of \textbf{client selection in FL} \cite{xia2020multi,cho2022towards,nishio2019client}. 
We adapt and modify prior work on FL convergence to reveal the relationships between model convergence and our definition of fairness across tasks. 

Past work on \textbf{multi-model federated learning} also aims to optimize the allocation of clients to model training tasks. \textit{To the best of our knowledge, we are the first to explicitly consider fairness across the different tasks}. 
Some algorithms focus on models with similar difficulties, 
for which fairness is easier to ensure~\cite{bhuyan2022multiProofs}, while other works propose 
reinforcement learning~\cite{liu2022multi}, heuristic~\cite{li2021efficient}, or asynchronous client selection algorithms~\cite{chang2023asynchronous,askin2024fedast} to optimize average model performance. Other work optimizes the training hyperparameters~\cite{nguyen2021toward} and allocation of communication channels to tasks~\cite{xia2020multi}. 

\textbf{Fair allocation of multiple resources} 
has been considered extensively in computer systems and networks ~\cite{poullie2017survey, joe2013multiresource, altman2008generalized}. 
In crowdsourcing frameworks, the fair matching of clients to tasks 
has been studied in 
energy sensing~\cite{lakhdari2021fairness}, delivery~\cite{basik2018fair}, and social networks~\cite{gan2017social}. 
However, unlike MMFL tasks, individual crowdsourcing tasks generally do not need multiple clients. 
Fair resource allocation across tasks in FL differs from all the above applications as the amount of performance improvement realized by allocating one (or more) clients to a given model's training task is unknown and stochastic, dependent on the model state and clients' data distributions.
In FL, fairness has generally been considered from the \textit{client}
perspective, e.g., accuracies or contribution 
when training a single model~\cite{li2019fair, donahue2021models}, \textcolor{blue}{and not across distinct training tasks.}

Many \textbf{incentive mechanisms} have been proposed for FL and other crowdsourcing systems~\cite{zhan2021survey, zhan2020learning, weng2021fedserving, wang2022dynamic, tang2021incentive}, including auctions for clients to participate in many FL tasks~\cite{deng2021fair}.
To the best of our knowledge, this is the first time client incentives 
have aimed to ensure fair training of multiple tasks in MMFL. 


\section{Fair Client Task Allocation for Multiple-Model FL}\label{sec:formulation}
\textcolor{blue}{We aim to train multiple models \textcolor{brown}{(i.e., tasks)}\footnote{\textcolor{brown}{We use model and task interchangeably in this paper.}} concurrently,
such that there is \textit{fair} training performance across models, e.g., in terms of the time taken, or converged accuracies. 
Models may be trained on distinct or the same datasets.
This performance is a function of which clients train which model at which training round.
In this section, we assume that clients \textcolor{orange}{have already committed to training all models,} 
and propose a client-model (i.e., client-task) allocation scheme that enables fair training of the multiple tasks.  
In Sec.~\ref{sec:incentiveAware}, we introduce auction mechanisms, \textcolor{violet}{which incentivize clients to commit to tasks at a fair rate across tasks,} 
before training begins (Fig.~\ref{fig:MMFLdiag}).
}%
\begin{figure}[t]
\centering
\includegraphics[scale=0.255]{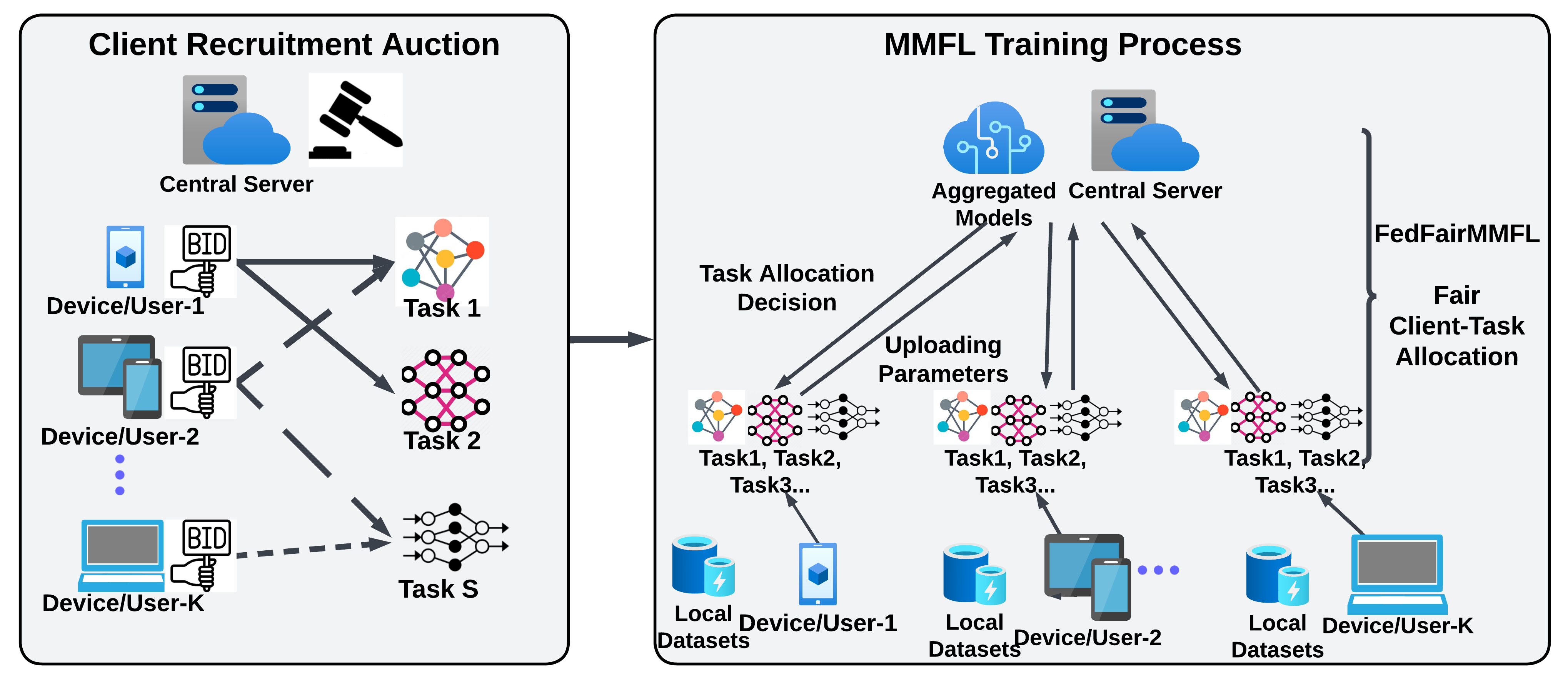}
  \caption{Fair Multiple Model Federated Learning to ensure that each model achieves comparable training performance: The pipeline's first part involves \textit{incentive-aware} client recruitment auctions to ensure a fair take-up rate over tasks. This is followed by a \textit{difficulty-aware} client-task allocation scheme that allocates tasks to clients based on the current loss levels of all tasks.
  } 
  \label{fig:MMFLdiag}
  \addtolength{\textfloatsep}{-0.2in}
  \vspace{-5mm}
\end{figure}

\noindent\textbf{MMFL framework:} Consider $K$ clients collaboratively training $S$ distinct machine learning tasks in MMFL (Fig. \ref{fig:MMFLdiag}). \textcolor{pink}{We let the set of tasks be $\mathbb{S}$. }
If $S = 1$, our setting reduces to traditional FL. Each client is equipped with local data for training each task. 
All clients can communicate with a central server that maintains a global model $\textbf{w}_s$ for each task \textcolor{red}{$s \in \mathbb{S}$.} 
Let each client $k$'s local loss over its local dataset $\mathcal{D}_{k,s}$ for task $s$ be $F_{k,s}(\textbf{w}_s) = \sum_{d\in\mathcal{D}_{k,s}} \ell\left(d, \textbf{w}_s\right)/\left|\mathcal{D}_{k,s}\right|$, where $\ell$ is a loss function that measures the performance of the model $\textbf{w}_s$ on data point $d$, and $\left|\mathcal{D}_{k,s}\right|$ is the size of client $k$'s dataset for task $s$. We then define the global loss for task $s$ \textcolor{violet}{as a weighted sum over the local losses:} 
\begin{equation}
f_s(\textbf{w}_s)=\sum_{k=1}^K p_{k,s} F_{k,s}(\textbf{w}_s),
\end{equation}
where $p_{k,s} = \frac{\left|\mathcal{D}_{k,s}\right|}{\sum_k \left|\mathcal{D}_{k,s}\right|}$ is the fraction of training data for task $s$ that resides at client $k$. Similar to traditional FL, in MMFL, the usual aim is to minimize the sum of losses~\cite{chang2023asynchronous}:
\begin{equation}
    \min_{\textbf{w}_1,\textbf{w}_2, ..., \textbf{w}_S} \sum_{s=1}^S f_{s}(\textbf{w}_s),
\end{equation}
%
\textbf{Fairness over tasks:} We introduce \textit{fairness} into this objective by utilizing the celebrated $\alpha$-fair utility function often employed to ensure fair network resource allocation~\cite{lan2010axiomatic,altman2008generalized}. Viewing the loss of each model as an inverse ``utility'', we then define our $\alpha$-fair MMFL objective as
\begin{equation}\label{eq:mmfl_obj}
    \min_{\textbf{w}_1,\textbf{w}_2, ... \textbf{w}_S} g\left(\textbf{w}_1,\textbf{w}_2, ... \textbf{w}_S\right) \equiv \sum_{s=1}^S f_{s}^{\alpha}(\textbf{w}_s),
\end{equation}
where $f_{s}^{\alpha}$ is task $s$'s $\alpha$-fair global loss function. This objective function helps to improve the uniformity of the converged losses, as we will analytically show in Section \ref{sec:fairConvgProofs}. 
Here the parameter $\alpha\in [1,\infty)$ controls the emphasis on fairness; for $\alpha = 1$, the objective is to minimize the average global loss across tasks. As $\alpha$ grows, we put more emphasis on fairness,\footnote{Usually, $\alpha$-fair functions aim to \textit{maximize} a sum of utility functions $(1 - \alpha)^{-1}\sum_s U_s^{1 - \alpha}$, where $U_s$ is the utility of task $s$. We instead aim to \textit{minimize} the sum of $\alpha$-fair loss functions, viewing the loss as an \textit{inverse} utility.}
\textcolor{violet}{because analytically, a higher $\alpha$ will give higher priority to the terms in the objective with larger $f_s$, during minimization.}

To solve (\ref{eq:mmfl_obj}) within the FL framework, we will 
design a training algorithm given the gradient 
\begin{equation} 
\label{eq:grad_g}
 \nabla_{\textbf{w}_s} g = \alpha f_{s}^{\alpha-1}(\textbf{w}_s) \sum_{k=1}^K p_{k,s} \nabla F_{k,s}(\textbf{w}_s).
\end{equation}

\noindent\textbf{FedFairMMFL training (Algorithm~\ref{algo:FedFairMMFL}):} Our MMFL training begins with the initialization of the global model weights \(\textbf{w}_s\) for all tasks, along with an initial random client-task allocation (line 1). The algorithm then enters a loop over global rounds (lines 2-16). During each global round $t$, 
\(m = C \times K\) \textcolor{brown}{clients are active} (line 3). Here $C$ is the \textcolor{brown}{rate at which clients are active.} 
\textcolor{magenta}{Clients can be inactive due to reasons such as low or unreliable connectivity, devices sleeping to conserve battery, or the lack of incentives on the part of clients.}
For all active clients, \textcolor{brown}{task allocation happens.}
\textcolor{magenta}{We let $Sel_s(t)$ denote the set of clients allocated to task $s$ in round $t$.} 
Following prior work on MMFL \cite{bhuyan2022multiProofs, bhuyan2022multi,liu2022multi}, we \textit{allocate each client to at most one task during each training round}, i.e., the $Sel_s(t)$ are disjoint across tasks $s$. For example, mobile clients \textcolor{violet}{or other resource constrained devices} may be performing other background tasks, e.g., running multiple applications, and thus may not have the computing power or battery to complete training rounds of multiple tasks. 

Intuitively, uniform client-task allocation methods, like random or round robin allocation~\cite{bhuyan2022multi}, will train simpler tasks faster, since they allocate an equal number of clients to each task, on expectation.
We would instead like to ensure tasks are trained at more similar rates (i.e., achieve more fair performance) with a \textbf{difficulty-aware} client-task allocation algorithm, in which more users are allocated to tasks that currently have a higher loss compared to other tasks.
We can do so by inspecting the gradient of our MMFL objective in (\ref{eq:grad_g}). Since \textcolor{violet}{resource-constrained} clients can only compute gradients for one task per training round, we cannot exactly compute these gradients in (\ref{eq:grad_g}). Instead, we will assign each client to contribute to the gradient $\nabla F_{k,s}(\textbf{w}_s)$ for each task $s$ 
\textcolor{black}{at a probability in proportion to its weight for that task}, i.e., proportional to $\alpha f_s(\textbf{w}_s)^{\alpha - 1}$. 
We therefore allocate a task \textcolor{magenta}{to each active client} according to the probability vector \textbf{p}
\textcolor{brown}{(lines 4-5)}:
\begin{equation}
    \textbf{p}=\left[\frac{f^{\alpha-1}_1}{\sum_{s\in \mathbb{S}} f^{\alpha-1}_s}, \frac{f^{\alpha-1}_2}{\sum_{s\in \mathbb{S}} f^{\alpha-1}_s}, ..., \frac{f^{\alpha-1}_S}{\sum_{s\in \mathbb{S}} f^{\alpha-1}_s} \right],
    \label{eq:TaskAllocProbVec}
\end{equation}

This allocation fulfills our intuition: \textbf{tasks currently with higher losses $f_s$ have a higher probability of being allocated clients.} 
\textcolor{blue}{For example, when $\alpha \rightarrow \infty$, \textit{all clients} will be allocated to the worst performing task, and when $\alpha=1$, allocation will be uniform.}
\textcolor{orange}{Our algorithm does not assume that the task with larger model is the more difficult task, rather it takes the current performance (loss) level $f_s(\textbf{w}_s)$ as a signal on which task needs more clients currently.}
In Section~\ref{sec:fairConvgProofs}, we formalize this intuition by deriving a convergence bound for each task using \textbf{FedFairMMFL}. 
This selection scheme is \textit{unbiased} across clients: \textcolor{pink}{
for each specific task, every client has the same probability of being selected. }

\textcolor{orange}{\textbf{Normalization:} Different models' loss values may belong to different ranges, 
making them difficult to compare. 
Therefore, in practice, we ensure comparable values by instead using 
the accuracy percentages, for all tasks, as an estimate of $f_s(\textbf{w}_s(t))$, to obtain Eq. (\ref{eq:TaskAllocProbVec}).}

Given the clients allocated to each task $s$ at time $t$, $Sel_s(t)$, the server then sends the latest global model weights  $\textbf{w}_s(t)$ to task $s$'s allocated clients (lines 7). Each client then trains its task on its corresponding local dataset by running $\tau$ SGD steps on the local loss $F_{k,s}$ of its assigned task $s$, as in traditional FL~\cite{mcmahan2017communication} (line 8-9). Let $\textbf{w}_{k,s}(t)$ denote the updated local model for active client $k$ on its allocated task $s$.
Each active client then uploads $\textbf{w}_{k,s}(t)$ to the server (line 10), which aggregates the model parameters to synchronize the global model. This aggregation uses weights 
\textcolor{orange}{
\begin{equation}
p_{k,\textcolor{magenta}{Sel_s}} = p_{k,s}/\sum_{k'\in \textcolor{magenta}{Sel_s}} p_{k',s},
\end{equation}}
\textcolor{orange}{which indicate the fraction of the training data across selected clients that resides at client $k$}, for each active client $k$ (line 13). \textcolor{magenta}{$Sel_s$} is the set of clients training task $s$ in this training round.
Hence, $p_{k,\textcolor{magenta}{Sel_s}}$ is a function of $p_{k,s}$.
\textcolor{brown}{The server then obtains the global loss for each task $f_s$, to guide allocation in the next training round (line 14).}
The next training round then begins.

\textcolor{orange}{\textbf{Reducing overheads:} To reduce communication overhead, 
in \textbf{FedFairMMFL}, we can instead let the central server 
evaluate the global loss $f_s(\textbf{w}_s)$ on its test set, for client task allocation. 
This reduces communication and computation (inference) overhead each global round.}

\textcolor{magenta}{\textbf{Distribution of active clients:} 
For simplicity, we assume that clients have the same active rate.
Our algorithm can also be applied for the setting where clients have different active rate from each other and over time:
given the active clients each global round, our algorithm will allocate each active client to a task, according to the probability vector \textbf{p} (Eq. (\ref{eq:TaskAllocProbVec})).}

\textcolor{orange}{In Table \ref{tab_parameters}, we present the various parameters of our system model.}

\begin{table}[!htbp]
	\caption{Definition of Notations}
	\centering
		\setlength{\tabcolsep}{0.3mm}{
			\begin{tabular}{|c|l|}
				\hline
				\hline
				\textcolor{orange}{Notations} & \textcolor{orange}{Definitions}  \\
				\hline
    $\mathbb{S}= \{1,2,... ,s, ..., \} $ & The set of all FL tasks.\\
    \hline
    \textcolor{magenta}{$K$} & \textcolor{magenta}{The total number of clients.} \\
    \hline
$ \mathbf{w}_s $ 
&  Global model for task $s$.\\
    \hline
				$ \mathcal{D}_{k,s}$ & Client $k$'s local dataset for task $s$. \\
    \hline
				$F_{k,s}(\textbf{w})$ & 
    Client $k$'s local loss for task $s$, given model $\textbf{w}$. \\ 
    \hline
				$ \ell \left(d, \textbf{w}_s\right)$ & The loss of model $\textbf{w}_s$ on datapoint $d$.\\
    \hline
				$f_s(\textbf{w}_s)$ & Global loss for task $s$.\\
    \hline
				$ p_{k,s}$ & 
    Fraction of task $s$'s data residing at client $k$.\\ 
    \hline
    $f_s^{\alpha}(\textbf{w}_s)$ & $\alpha$-fair global loss function for task $s$.\\
    \hline
				$g(\textbf{w}_1 , \textbf{w}_2, ...)$ & Sum of the $\alpha$-fair global loss function, across tasks.\\
                \hline
                \textcolor{magenta}{$\mathbf{p}$} & \textcolor{magenta}{Client selection probability vector.} \\
    \hline
				\textcolor{magenta}{$ Sel_s(t)$} 
                & Subset of clients allocated to task $s$ at round $t$.\\
    \hline
    $\textbf{w}_{k,s}(t)$ & Local model, client $k$, task $s$, at time $t$. \\
    \hline
    $p_{k,Sel_s}$ & \begin{tabular}[c]{@{}l@{}}  Proportion of the data which client $k$ has, amongst \\ the selected clients in \textcolor{magenta}{$Sel_s$}. \end{tabular} \\
    \hline
    $a_{k,s}(t)$ & Variable indicating if client $k$ trains task $s$.\\
    \hline
    $B^s_{Sel_s}(\alpha)$ & \begin{tabular}[c]{@{}l@{}} 
 Probability that a specific set of clients $Sel_s$ is \\  selected for task $s$ by our algorithm. See Eq. (\ref{eq:B_s_Sel}). \end{tabular} \\
 \hline
 $\textbf{w}_s^*$& Global optimal. \\
 \hline
 $\textbf{w}_{k,s}^* $ & Client $k$'s optimal.\\
 \hline
 $\Gamma_s$ &  \begin{tabular}[c]{@{}l@{}} The gap between the loss evaluated at the global\\ optimum $\mathbf{w}_s^*$, and the loss evaluated at  $\mathbf{w}^*_{k,s}$, \\ for task $s$. See Eq. (\ref{eq:gamma_s}). \end{tabular}\\
				\hline
   $\rho^{\alpha}(B^s_{Sel_s}, \textbf{w}' )$ & \begin{tabular}[c]{@{}l@{}} Selection skew, the ratio of the expected 'local-\\global objective gap' of our algorithm's selected\\ clients, over the 'local-global objective gap' of all\\ $K$ clients. See Eq. (\ref{eq:selectionSkew}). \end{tabular} \\
   \hline
   $\overline{\rho}^{\alpha}$, $\underline{\rho}^{\alpha}$ & Upper and lower bounds on the selection skew.\\
   \hline 
   $\sigma^2$ & Bound on the variance of stochastic gradient.  \\
   \hline 
   $(\mathbf{w}^*_{1_{\alpha}}, ... \mathbf{w}^*_{S_{\alpha}})$
   & \begin{tabular}[c]{@{}l@{}} The global optimal solution when parameter $\alpha$ is \\ used. \end{tabular} \\
   \hline
		\end{tabular}}
	\label{tab_parameters}
\end{table}

\begin{algorithm}[t]
	\caption{\textcolor{red}{FedFairMMFL: Fair Training of Multiple Models in Federated Learning}}
 \label{algo:FedFairMMFL}
 \small
	\begin{algorithmic}[1]
        \State \textbf{Initialization:} The global model weights $\textbf{w}_s$, for all tasks, \textcolor{violet}{initial random client-task allocation.}
		\For {$Global~ epoch ~t=1,2,\ldots T$}
                \State 
                $m = C \times K$ clients \textcolor{brown}{are active.}
                \State \textcolor{brown}{Server determines task allocation for active clients,} 
            \State \textcolor{brown}{based on $\mathbf{p}=\left[{f^{\alpha-1}_1}/{\sum_{s\in \mathbb{S}} f^{\alpha-1}_s}, ... \right]$ (Eq. \ref{eq:TaskAllocProbVec}).}
                \For{all \textcolor{brown}{active} clients $k$} 
                    \State Server sends the latest model weights of task $s$
                    ($\textbf{w}_s$).
                    \State Client performs local training for $\tau$ local rounds for
                    \State allocated task, optimizing $\textbf{w}_{k,s}$.
                    \State Client uploads updated $\textbf{w}_{k,s}$ to the server.
			\EndFor
			\For{all tasks $s \in \mathbb{S}$} 
                    \State Server aggregates 
                   $\textbf{w}_s \gets \sum_{k\in Sel_s}  p_{k,Sel_s}  \textbf{w}_{k,s} $. 
                    \State \textcolor{brown}{Obtain global accuracy $f_s$ for task $s$. }
                \EndFor
           \EndFor
	\end{algorithmic} 
\end{algorithm}


\section{Fairness and Convergence Guarantees}
\label{sec:fairConvgProofs}

We now analyze \textbf{FedFairMMFL}'s performance. \textcolor{red}{The full versions of all our proofs are in our Technical Report \cite{MMFLappendix}.}

\subsection{Fairness Guarantees}
We first analyze the optimal solution to our fairness objective (\ref{eq:mmfl_obj}) to show that it reduces the variation in training accuracy, which intuitively improves fairness across tasks. 
Firstly, we let \textcolor{orange}{ 
$G_\alpha$ represent the optimal value of our MMFL fairness objective (\ref{eq:mmfl_obj}), given that each client only trains one task each global epoch}. 
\begin{equation}
    G_\alpha =\min_{\textcolor{orange}{\mathbf{w}_1, ... \mathbf{w}_S}}   \sum^S_{s=1} f^{\alpha}_s (\textbf{w}_{s}) \quad
    \text{s.t.} \sum_s a_{k,s}(t)\leq 1\ \forall k\ \forall t ,
    \label{eq:def_w_s_alpha}
\end{equation}
\textcolor{blue}{where \textcolor{brown}{$a_{k,s}(t)=1$ indicates that client $k$ trains task $s$ during epoch $t$, and} $\sum_s a_{k,s}(t)\leq 1\ \forall k\ \forall t$ expresses the constraint that each client trains 1 task each epoch \textcolor{pink}{at most}. 
\textcolor{orange}{We let 
\((\mathbf{w}^*_{1_{\alpha}}, ... \mathbf{w}^*_{S_{\alpha}})\)
} 
denote the global optimal solution \textcolor{pink}{for \eqref{eq:def_w_s_alpha}} when parameter $\alpha$ is used.}  

To motivate the choice of $\alpha$, we show that for $\alpha = 1$ and $2$, a higher $\alpha$ puts more emphasis on fairness: we relate this $\alpha$-fairness metric to the well-known \textit{variance} and \textit{cosine similarity} metrics for similarity across task accuracies. 
\textcolor{magenta}{Note that the specific case of $\alpha=2$ is equivalent to the square operation on $f_s$.}

\begin{lemma}[Fairness and Variance]\label{lemma:variance} $\alpha=2$ yields a smaller variance in the task \textcolor{orange}{losses} 
than $\alpha=1$ does:
\begin{equation*}
\textbf{Var}(f_1(\textbf{w}_{1_2}^*), ..., f_S(\textbf{w}_{S_2}^*)) \leq \textbf{Var}(f_1(\textbf{w}_{1_1}^*), ..., f_S(\textbf{w}_{S_1}^*)). 
\label{eq:VarTheoremEqn}
\end{equation*}
\textcolor{magenta}{The left hand side of this equation refers to the variance in the task losses, at the optimal solution of $G_{\alpha}$, when $\alpha=2$. The right hand side of this equation refers to the variance in the task losses, at the optimal solution of $G_{\alpha}$, when $\alpha=1$.}
\end{lemma}
\begin{proof}
\begin{align}
    &\textbf{Var}(f_1(\mathbf{w}_{1_2}^*), ..., f_n(\mathbf{w}_{n_2}^*))\\
    &= \frac{\sum_{s=1}^S f^{\textcolor{magenta}{2}}_s(\mathbf{w}_{s_2}^*)}{S} - \bigg(  \sum_{s=1}^S f^{\textcolor{magenta}{1}}_s (\mathbf{w}_{s_2}^*)\bigg)^2\\
    & \leq \frac{\sum_{s=1}^S f^{\textcolor{magenta}{2}}_s (\mathbf{w}_{s_1}^*)}{S}- \bigg(  \sum_{s_1}^S f^{\textcolor{magenta}{1}}_s (\mathbf{w}_{s_2}^*)\bigg)^2\\
    & \leq \frac{\sum_{s=1}^S f^{\textcolor{magenta}{2}}_s (\mathbf{w}_{s_1}^*)}{S} - \bigg(  \sum_{s=1}^S f^{\textcolor{magenta}{1}}_s (\mathbf{w}_{s_1}^*)\bigg)^2. 
\end{align}
where the first inequality follows from the fact that 
$(\mathbf{w}^*_{1_{2}}, ... \mathbf{w}^*_{S_{2}})$ is the optimal solution of $G_{2}$ \textcolor{magenta}{(Eq. (\ref{eq:def_w_s_alpha}))}, which is $\min_{
\mathbf{w}_1, ... \mathbf{w}_S} \bigg\{ \bigg( \sum^S_{s=1} f^{\textcolor{magenta}{\alpha=2}}_s (\mathbf{w}_{s}) \bigg), \text{s.t.} \sum_s a_{k,s}\leq1\ \forall k\ \forall t \bigg\} $, and the second inequality follows from the fact that $(\mathbf{w}^*_{1_{1}}, ... \mathbf{w}^*_{S_{1}})$ 
is the optimal solution of $G_1$ \textcolor{magenta}{(Eq. (\ref{eq:def_w_s_alpha}))}, which is $\min_{
\mathbf{w}_1, ... \mathbf{w}_S} \bigg\{ \bigg( \sum^S_{s=1} f^{\textcolor{magenta}{\alpha=1}}_s (\mathbf{w}_{s}) \bigg), \text{s.t.} \sum_s a_{k,s}\leq1\ \forall k\ \forall t \bigg\} $.
\end{proof}



\textcolor{black}{Next, we use the definition of cosine similarity to show that $\alpha=2$ results in a fairer solution than $\alpha=1$ (non-fairness).}
\textcolor{orange}{A higher cosine similarity between two quantities indicates that the quantities are more uniform in value.}
\begin{lemma}[Fairness and Cosine Similarity]
    $\alpha=2$ results in a higher cosine similarity across tasks 
    than $\alpha=1$,
i.e., 
\begin{equation}
   \frac{\sum_{s=1}^S f_s(\textbf{w}_{s_2}^*)}{\sqrt{\sum_{s=1}^S f_s^{\textcolor{magenta}{2}}(\textbf{w}_{s_2}^*)}} \geq \frac{\sum_{s=1}^S f_s(\textbf{w}_{s_1}^*)}{\sqrt{\sum_{s=1}^S f_s^{\textcolor{magenta}{2}}(\textbf{w}_{s_1}^*)}}.
\end{equation}
\end{lemma}
\begin{proof}
Since 
$(\mathbf{w}^*_{1_{1}}, ... \mathbf{w}^*_{S_{1}})$ is the optimal solution of $G_1=\min_{\mathbf{w}_1, ... \mathbf{w}_S} \bigg\{ \bigg( \sum^S_{s=1} f^{\textcolor{magenta}{\alpha=1}}_s (\mathbf{w}_{s}) \bigg), 
\text{s.t.} \sum_s a_{k,s}\leq 1\ \forall k\ \forall t \bigg\} $ 
\textcolor{magenta}{according to (\ref{eq:def_w_s_alpha})}, 
we have 
\begin{equation}
        \sum_{s=1}^S f_s^{\textcolor{magenta}{1}}(\mathbf{w}_{s_2}^*)\geq  \sum_{s=1}^S f_s^{\textcolor{magenta}{1}}(\mathbf{w}_{s_1}^*).
\end{equation}
Since $(\mathbf{w}^*_{1_{2}}, ... \mathbf{w}^*_{S_{2}})$ 
is the optimal solution of $G_{2}=\min_{\mathbf{w}_1, ... \mathbf{w}_S} \bigg\{ \bigg( \sum^S_{s=1} f^{\textcolor{magenta}{\alpha=2}}_s (\mathbf{w}_{s}) \bigg),
\text{s.t.} \sum_s a_{k,s}\leq 1\ \forall k\ \forall t \bigg\} $ \textcolor{magenta}{according to (\ref{eq:def_w_s_alpha})}, 
we have
\begin{equation}
    \sum_{t=1}^S f^{\textcolor{magenta}{2}}_s(\mathbf{w}_{s_1}^*) \geq  \sum_{t=1}^S f^{\textcolor{magenta}{2}}_s (\mathbf{w}_{s_2}^*).
\end{equation}
Therefore, we have
\begin{equation}
\frac{\sum_{s=1}^S f_s(\textbf{w}_{s_2}^*)}{\sqrt{\sum_{s=1}^S f_s^2(\textbf{w}_{s_2}^*)}} \geq \frac{\sum_{s=1}^S f_s(\textbf{w}_{s_1}^*)}{\sqrt{ \sum_{s=1}^S f_s^2(\textbf{w}_{s_1}^*)}}.
\end{equation}
\end{proof}

\subsection{Convergence Guarantees}
Here, we analyze the convergence of training tasks under our client-task allocation algorithm, showing that it preferentially accelerates training of tasks with higher losses and that each task will converge.
\textcolor{brown}{We prove convergence by taking the perspective of a single task in MMFL, 
without loss of generality.}

\textcolor{black}{We first characterize $B_{Sel_s}^s(\alpha)$, the \textbf{probability that a specific set of clients} \textcolor{magenta}{$Sel_s$} is \textcolor{brown}{selected for} a specific task $s$ by our algorithm \textbf{FedFairMMFL}.}
\begin{equation}
B_{Sel_s}^s(\alpha)=\textstyle   \left(\frac{f_s^{\alpha}(t)}{\sum_{s' \in S} f_{s'}^{\alpha}(t)}\right)^{|Sel_s|} \left(1- \frac{f_s^{\alpha}(t)}{\sum_{s' \in S} f_{s'}^{\alpha}(t)}\right)^{K-|Sel_s|} 
\label{eq:B_s_Sel}
\end{equation}
\textcolor{black}{$B_{Sel_s}^s(\alpha)$ follows the client-task allocation probability vector (Eq. (\ref{eq:TaskAllocProbVec})), and depends on the losses of all tasks $s \in S$, and our $\alpha$-fair parameter $\alpha$. 
}
\textcolor{black}{Next we introduce $\Gamma_s$, the \textit{gap} between the loss evaluated at the global optimum $\mathbf{w}_s^*$, and the loss evaluated at client $k$'s optimal $\mathbf{w}^*_{k,s}$, for task $s$, \textcolor{magenta}{summed over all clients}:}
\begin{align}
   \Gamma_s &= \sum_{k=1}^K p_{k,s} (F_{k,s} (\mathbf{w}_s^*) -F_{k,s}(\mathbf{w}^*_{k,s})) \nonumber \\
   &= f_s^* - \sum_{k=1}^K p_{k,s} F_{k,s}^*. 
   \label{eq:gamma_s}
\end{align}
\textcolor{black}{We further characterise the \textit{selection skew} $\rho^{\alpha}(B^s_{Sel_s},\mathbf{w}')$. 
The selection skew is a ratio of the expected `local-global objective gap' of the selected clients under \textbf{FedFairMMFL} 
over the `local-global objective gap' of all $K$ clients, where $\textbf{w}'$ is the point at which $F_{k,s}$ and $f_s$ are evaluated.
This skew depends on the expectation of client allocation combinations, and hence on $B_{Sel_s}^s$ and the $\alpha$-fair parameter $\alpha$:} 
\begin{align}
    &\rho^{\alpha}(B^s_{Sel_s},\mathbf{w}')
\label{eq:selectionSkew}\\ 
&=\frac{\sum_{Sel_s\in\mathbb{S}} \textcolor{blue}{B_{Sel_s}^s(\alpha)} \sum_{k \in Sel_s} \textcolor{blue}{\frac{p_{k,s} (F_{k,s}(\textbf{w}'_{s})-F_{k,s}^*)}{\sum_{k' \in Sel} p_{k’,s} }}
 }{f_s(\mathbf{w}')- \sum_{k=1}^K p_{k,s} F_{k,s}^* }\nonumber
\end{align}
We use $\overline{\rho}^{\alpha}$ and $\tilde{\rho}^{\alpha}$ respectively to denote lower and upper bounds on the selection skew throughout the training.
Recall that $\tau$ denotes the number of SGD epochs in local training.
We then introduce the following assumptions used in our convergence analysis, which are standard in the FL literature~\cite{cho2022towards}.

\textbf{Assumption 1:} 
\textit{
$F_{1,s}, .., F_{k,s}$ are L-smooth, i.e.} 
\begin{equation}
F_{k,s}(\mathbf{v}) \leq F_{k,s}(\mathbf{w})+ (\mathbf{v}-\mathbf{w})^T \nabla F_{k,s}(\mathbf{w}) + \frac{L}{2} \| \mathbf{v}-\mathbf{w}\|^2,
\end{equation}
\haoran{for all $\mathbf{v}$ and $\mathbf{w}$.}

\textbf{Assumption 2:} 
\textit{ 
$F_{1,s}, .. F_{k,s}$ are $\mu$-strongly convex, i.e.} 
\begin{equation}
F_{k,s}(\mathbf{v}) \geq F_{k,s}(\mathbf{w})+ (\mathbf{v}-\mathbf{w})^T \nabla F_{k,s}(\mathbf{w}) + \frac{\mu}{2} \| \mathbf{v}-\mathbf{w}\|^2, 
\end{equation}
for all $\mathbf{v}$ and $\mathbf{w}$.
Note that $\mu \leq L$ from Assumption 1. 

\textbf{Assumption 3:} \textit{For all clients $k$, for the minibatch $\xi_{k,s}$ uniformly sampled at random, the stochastic gradient is unbiased:} 
\begin{equation}
\mathbb{E}[g_{k,s} (\mathbf{w}_{k,s} , \xi_{k,s})] =\nabla F_{k,s}(\textbf{w}_{k,s}), 
\end{equation}
\textit{and its variance is bounded:} 
\begin{equation}
\mathbb{E}\|g_{k,s} (\mathbf{w}_{k,s} , \xi_{k,s})- \nabla F_{k,s}(\textbf{w}_{k,s}) \|^2 \leq \sigma^2.
\end{equation}

\textbf{Assumption 4:} \textit{Bounded stochastic gradient. For all clients $k$, the expected square norm of the stochastic gradient is bounded:} 
\begin{equation}
\mathbb{E}\| g_{k,s} ( \textbf{w}_{k,s}, \xi_{k,s}) \|^2 \leq G^2.
\end{equation}

\textcolor{orange}{Next, we introduce Lemma \ref{lemma:globalWsAndlocalWs_Gap}, which bounds the expected average discrepancy between the global model for task $s$, and client $k$'s model for task $s$, under client selection set \textcolor{magenta}{$Sel_s$}.}
\begin{lemma}[Local and Global Model Discrepancy]
    The expected average discrepancy between \textcolor{blue}{the global model for task $s$,} $\overline{\textbf{w}}^{(t)}_s$ and \textcolor{blue}{client $k$'s model for $s$,} $\textbf{w}_{k,s}^{(t)}$, under selection $k \in \textcolor{magenta}{Sel_s(t)}$, can be bounded as 
    \begin{align}
        &\mathbb{E}[\sum_{k \in \textcolor{magenta}{Sel_s(t)}} p_{k,\textcolor{magenta}{Sel_s}} 
        \| \overline{\textbf{w}}^{(t)}_s - \textbf{w}_{k,s}^{(t)} \|^2 ] \\
        & \leq  16 \eta_t^2 \tau^2 G^2 \sum_{\textcolor{magenta}{Sel_s(t)} \in \mathbb{S}(t)} B_{Sel_s(t)}^s (\alpha) P^{k,2}_{\textcolor{magenta}{Sel_s(t)}} 
    \end{align}
where $P^{k,2}_{Sel_s(t)}= \sum_{\substack{k \neq k', \\ k, k' \in \textcolor{magenta}{Sel_s(t)} }} p_{k,s} \frac{p_{k',s}}{(\sum_{k^{''} \in \textcolor{magenta}{Sel_s(t)}} p_{k'',s})^2}$. 
\label{lemma:globalWsAndlocalWs_Gap}
\end{lemma}
\textcolor{orange}{To prove Lemma \ref{lemma:globalWsAndlocalWs_Gap}, we modified the proof in \cite{cho2022towards}, to incorporate the probability that a specific set of clients $Sel$ is allocated to the task, under our client selection scheme, and we sum over all possible sets of clients.}
\textcolor{red}{The full proof is available in our Technical Report \cite{MMFLappendix}.}

\textcolor{brown}{We use the above result in Lemma \ref{lemma:globalWsAndlocalWs_Gap}, to help us prove the next theorem. In the next theorem, without loss of generality we show the convergence rate of a task in the MMFL setting} 
under \textbf{FedFairMMFL}, and show its dependence on the relative loss levels of \textbf{all} tasks. When the prevailing loss level of task $s$ is higher, its convergence will be sped up under our algorithm.

\begin{theorem}[Model Parameter Convergence]
Suppose we use \textcolor{magenta}{under a decaying learning rate $\eta_t = \frac{1}{\mu (t+\gamma)}$} 
at each round $t$. 
Under Assumptions 1-4, the progress of task $s$'s global parameters at round $t + 1$, $\overline{\textbf{w}}_s^{(t + 1)}$, towards the optimal parameters $\textbf{w}^*$: $\mathbb{E}\left[\| \overline{\textbf{w}}^{(t+1)}_s-\textbf{w}_s^* \|^2 \right]$ is upper-bounded by
\begin{align}\label{eq:w_bound}
    &\left[1-\eta_t \mu (1+ \frac{3 \overline{\rho}^{\alpha} }{8})\right] \mathbb{E} [\| \overline{\textbf{w}}^{(t)}_s - \textbf{w}^*_s \|^2 ]  + 2\eta_t \Gamma_s (\tilde{\rho}^{\alpha}-\overline{\rho}^{\alpha}) + \notag \\
    & \eta_t^2 \sigma^2\sum_{\textcolor{magenta}{Sel_s(t)} \in \mathbb{S}(t)} \sum_{k\in \textcolor{magenta}{Sel_s(t)}}B_{\textcolor{magenta}{Sel_s(t)}}^s (\alpha) \left(\frac{p_{k,s}}{\sum_{k' \in \textcolor{magenta}{Sel_s(t)}} p_{k’,s}}\right)^2 \notag \\
    & + \eta_t^2 \left(32 \tau^2 G^2 +6 \overline{\rho}^{\alpha} L \Gamma_s \right). 
\end{align}
\label{theorem:convg_boundofOneItr}
\vspace{-4mm}
\end{theorem}
\textit{Proof Sketch:} \textcolor{violet}{We modified \cite{cho2022towards}'s proof to incorporate our distinct client selection scheme, which unlike \cite{cho2022towards}'s proof requires explicitly quantifying the impact of $\textcolor{magenta}{Sel_s(t)}$ and $B^s_{\textcolor{magenta}{Sel_s}}$ \textcolor{red}{(Eq. \ref{eq:B_s_Sel})} 
through various bounds on the loss function gradients.    
In the rest of our proof, we add and subtract multiple expressions involving the gradients, bound the various terms, 
and incorporate Lemma \ref{lemma:globalWsAndlocalWs_Gap}'s result.} 
\textcolor{red}{The full proof is available in our Technical Report \cite{MMFLappendix}.}
\qed

In tracking the impact of our client-task allocation on the convergence in Theorem~\ref{theorem:convg_boundofOneItr}'s proof, we also reveal how the convergence bound in~\eqref{eq:w_bound} depends on our client-task allocation algorithm, in particular the parameter $\alpha$ and the current task losses. 
\textcolor{black}{This result not only helps us 
to prove the convergence of our \textbf{FedFairMMFL} training algorithm, 
but also provides additional justification for our client-task allocation strategy: \textcolor{violet}{Mathematically, tasks that are further from the optimum (i.e., with a larger $\mathbb{E} [\| \overline{\textbf{w}}^{(t)} - \textbf{w}^*_s \|^2$) should be allocated clients to reduce the right-hand side of (\ref{eq:w_bound}) to equalize the progress across all tasks, which agrees with our algorithm.}
\textcolor{black}{In Corollary \ref{corollary:onAlphaAndBSelTerm}, we show that under certain conditions on $p_{k,s}$, as $\alpha$ increases, the term with $B^s_{\textcolor{magenta}{Sel_s}}(\alpha)$ decreases, hence controlling the task's progress towards its optimum.}}

\begin{corollary}[Fairness and Convergence]
\label{corollary:onAlphaAndBSelTerm}
Suppose that all of the $p_{k,s} = \frac{1}{K}$ and consider task $s$ that has the highest loss $f_s\left(\textbf{w}_s(t + 1)\right)$. Then as $\alpha$ increases, \textcolor{blue}{the term 
\begin{equation}
\textcolor{magenta}{\eta_t^2 \sigma^2\sum_{\textcolor{magenta}{Sel_s} \in \mathbb{S}(t)} \sum_{k\in \textcolor{magenta}{Sel_s(t)}}B_{\textcolor{magenta}{Sel_s}}^s (\alpha) \left(\frac{p_{k,s}}{\sum_{k' \in \textcolor{magenta}{Sel_s(t)}} p_{k’,s}}\right)^2}
\label{eq:termInBoundOfThr4}
\end{equation}
in the right-hand side of (\ref{eq:w_bound}) decreases.}
\end{corollary}
\begin{proof}
\textcolor{magenta}{We firstly define $\overline{f}_s(\alpha)$ as:
\begin{align}
\overline{f}_s(\alpha) = \frac{f_s\left(\textbf{w}_s(t + 1)\right)^\alpha}{\sum_{s'} f_{s'}\left(\textbf{w}_s(t + 1)\right)^\alpha}.
\end{align}}
We rewrite 
\begin{equation}
\textcolor{magenta}{\sum_{\textcolor{magenta}{Sel_s(t)} \in \mathbb{S}(t)} \sum_{k\in \textcolor{magenta}{Sel_s(t)}}B_{\textcolor{magenta}{Sel_s}}^s(\alpha)\left(\frac{p_{k,s}}{\sum_{k' \in S(t)} p_{k’,s}}\right)^2}
\end{equation}
as $\sum_{j = 1}^K \frac{1}{j}\left(\overline{f}_s(\alpha)\right)^j\left(1 - \overline{f}_s(\alpha)\right)^{K - j}{K \choose j}.$
Since $\sum_{j = 1}^K \left(\overline{f}_s(\alpha)\right)^j\left(1 - \overline{f}_s(\alpha)\right)^{K - j}{K \choose j} = 1$, we can view this as a weighted average of $1, \frac{1}{2}, \ldots, \frac{1}{K}$. It then suffices to show that, as $\alpha$ increases, we place more weight on smaller terms. Note that
\begin{equation}
\textcolor{magenta}{\frac{d}{d\alpha}\left(\overline{f}_s(\alpha)^j\left(1 - \overline{f}_s(\alpha)\right)^{K - j}\right) \propto \left(\frac{d \overline{f}_s(\alpha)}{\alpha}\right)\left(j - K\overline{f}_s(\alpha)\right)}
\end{equation}
It then suffices to show that $\frac{d \overline{f}_s(\alpha)}{\alpha} > 0$; if so, the terms with larger $j$ will increase as $\alpha$ increases. Hence,
$\frac{d \overline{f}_s(\alpha)}{\alpha} \propto \sum_{s' = 1}^S f_{s'}^\alpha \left(\log f_s - \log f_{s'}\right) > 0$
as we have assumed task $s$ has the highest loss $f_s$.
\end{proof}
According to Corollary \ref{corollary:onAlphaAndBSelTerm}, as the fairness parameter $\alpha$ increases, our client-task allocation algorithm ensures that the worst-performing task makes more progress towards its training goal, thus helping to equalize the tasks' performance and fulfilling the intuition of a higher $\alpha$ leading to more fair outcomes. The result further suggests that an alternative client-task allocation algorithm can explicitly enforce fairness across the upper-bounds in (\ref{eq:w_bound}). However, this approach requires integer programming methods and estimating unknown constants like $\sigma^2$ and $L$ for each task, making it less practical than \textbf{FedFairMMFL}. 

Finally, we provide the convergence bound for \textbf{FedFairMMFL}:
\begin{corollary}
    [Model Convergence] With Assumptions 1-4, under a decaying learning rate $\eta_t = \frac{1}{\mu (t+\gamma)}$ and under our difficulty-aware client-task allocation algorithm \textbf{FedFairMMFL}, the error after $T$ iterations of federated learning for a specific task $s$ in multiple-model federated learning is 
    \begin{align*}
   &\mathbb{E}[f_s(\overline{\textbf{w}}^{(T)}_{s})]-f_s^* \leq
        \frac{1}{T+\gamma } \bigg[ \frac{4[16 \tau^2 G^2 + \sigma^2]}{3 \overline{\rho}^{\alpha}\mu^2} + \frac{8 L^2 \Gamma_s}{\mu^2} \\
        & +\frac{L \gamma \| \overline{\textbf{w}}^{(0)}_{s} -\textbf{w}^*_{s} \|^2 }{2} \bigg] + \frac{8 L \Gamma_s}{3 \mu} (\frac{\tilde{\rho}^{\alpha}}{\overline{\rho}^{\alpha}}-1).
\end{align*}
\label{corollary:firstConvgBound}
\vspace{-3mm}
\end{corollary}
\textit{Proof Sketch.} \textcolor{violet}{
We use the upper bound of $\mathbb{E}\left[\| \overline{\textbf{w}}^{(t+1)}_s-\textbf{w}_s^* \|^2 \right]$
from Theorem \ref{theorem:convg_boundofOneItr}, which incorporates our client-task allocation scheme. Next, we upper bound 
\begin{equation}
\textcolor{magenta}{
\Delta_{t+1}=\sum_{\textcolor{magenta}{Sel_s(t)} \in \mathbb{S}} B_{Sel_s}^s(\alpha) P^{k,2}_{Sel_s}}
\end{equation}
by $\frac{1}{2}$.} Firstly, the inner summation $P^{k,2}_{Sel(t)}$ can be re-written as 
\begin{align*}
     & \sum_{k,k' \in Sel_s(t), k \neq k'}  \frac{p_{k,s}}{(\sum_{k^{''} \in Sel_s(t)} p_{k'',s})} 
      \frac{p_{k',s}}{(\sum_{k^{''} \in Sel_s(t)} p_{k'',s})}  \\
    &= \frac{(\sum_{k^{''} \in Sel_s(t)} p_{k'',s})^2 - \sum_{k\in Sel_s(t)} p_{k,s}^2}{2(\sum_{k^{''} \in Sel_s(t)} p_{k'',s})^2}\\ 
    & \leq \frac{1}{2}.
\end{align*}
Secondly, the outer summation $\sum_{Sel_s(t)} B_{Sel_s}^s(\alpha(t))$ is a binomial distribution with binomial parameter $f_s^\beta(t)/ {\sum_{s' \in S} f_{s'}^{\beta}(t)}$,
and this sums to 1.
\textcolor{violet}{Now we have a bound of $\Delta_{t+1}$ in terms of $\Delta_{t}$ and the learning rates $\eta_t$. Next, we use induction to bound $\Delta_{t+1}$ in terms of time independent variables. Finally, we use the L-smoothness of $F_s$ to obtain the result of this corollary. 
}
\textcolor{red}{The full proof is available in our Technical Report \cite{MMFLappendix}.}
\qed 

\textcolor{magenta}{As seen in the convergence bound, convergence to the optimum is impacted by the following parameters and effects:}

\textcolor{magenta}{\textit{First term:} Our upper bound on $\mathbb{E}[f_s(\overline{\textbf{w}}^{(T)}_{s})]-f_s^*$ is quadratically proportional to the number of local epochs $\tau$ and the bound of the stochastic gradient $G$, and increases linearly with the bound of the variance of the stochastic gradient $\sigma^2$. Intuitively, as the number of local epochs or the stochastic gradient's variance increase, convergence would be slower (i.e., our upper bound would be larger) due to greater client divergence and gradient noise, respectively. 
At the same time, our upper bound decreases in proportion to $\overline{\rho}^{\alpha}$, the lower bound on the selection skew. Intuitively, we can expect that as the selection skew increases, convergence will accelerate (i.e., our upper bound will be lower), as we have been selecting clients that were further from the global optimum with a higher probability. The convergence bound is also inversely proportional to $\mu^2$, which indicates that the higher the $\mu$-convexity, the faster the convergence.}

\textcolor{magenta}{\textit{Second term:} Our upper bound is proportional to the square of $L$, the Lipschitz constant, as well as to $\Gamma_s$, the gap between the loss evaluated at the global optimum $\textbf{w}_s^*$ and that evaluated at client $k$'s optimal $w_{k,s}^*$ for task $s$, summed over all clients. Intuitively, if the data distribution is more non-iid (meaning $\Gamma_s$ is larger, convergence will be slower, as shown by an increase in our upper bound. Convergence is inversely proportional to the square of $\mu$, which as for the first term indicates that the higher the $\mu$-convexity, the faster the convergence.}

\textcolor{magenta}{\textit{Third term:} As the discrepancy between the initial point of training and the global optimum for task $s$ increases, the model requires more iterations to approach the optimum.}

\textcolor{magenta}{\textit{Fourth term:} Similar to the second term, convergence is proportional to $L$, the Lipschitz constant, as well as to $\Gamma_s$, the gap between the loss evaluated at the global optimal $\textbf{w}_s^*$ and the loss evaluated at client $k$'s optimal $w_{k,s}^*$ for task $s$, summed over all clients. Convergence is inversely proportional to $\mu$, which indicates that the higher the $\mu$-convexity, the faster the convergence. The term $(\frac{\tilde{\rho}^{\alpha}}{\overline{\rho}^{\alpha}}-1)$ indicates the gap between the upper bound and lower bound of the selection skew, and hence the bias of the selection skew. As the bias increases, convergence slows down, i.e., our upper bound is larger, since the selection is biased to certain clients.}




\section{Client Incentive-Aware Fairness across Tasks}
\label{sec:incentiveAware}

In this section, we account for the fact that \textit{clients may have different willingness to train each model}. 
Prior incentive-aware FL work~\cite{zhan2021survey, tang2021incentive, deng2021fair} \textcolor{orange}{proposed auctions to incentivize clients to join FL training. }
These works 
aimed to maximize the total number of users performing 
either one (or many) tasks. 
\textcolor{orange}{However, in the multiple model case, when many users favor a particular task over others, the performance of the disfavored task can suffer, as fewer users may be incentivized to perform this task.}
Here, we would like to ensure a ``fair'' number of users incentivized across tasks, 
to further ensure fair training outcomes. To avoid excessive overhead during training, \textcolor{orange}{the client recruitment auctions in this section} 
occur before training begins~\cite{ruan2021valuable}. 
In this section, we refer to potential clients as `users', as they have not yet committed to the training.

\textcolor{blue}{
We model user $i$'s unwillingness towards training task $s$ by supposing $i$ will not train task $s$ unless they receive a payment $\geq c_{i,s}$. Following prior work~\cite{deng2021fair}, we suppose the $c_{i,s}$ are private and employ an auction that does not require users to explicitly report $c_{i,s}$; 
auctions require no prior information about users' valuations, 
which may be difficult to obtain in practice. 
Users will submit bids $b_{i,s}$ indicating the payment they want for each task. Given the bids, the server will choose the auction winners \textcolor{magenta}{who will participate in MMFL training (second part of pipeline in Fig. \ref{fig:MMFLdiag}),} 
and each user's payment.
User $i$'s utility for task $s$ will be}
\begin{equation}
u_{i,s}=
\begin{cases}
    p_{i,s}-c_{i,s} & \text{if they participate in this task}\\
    0 &\text{otherwise}
\end{cases}
\end{equation}
\textcolor{blue}{where $p_{i,s}$ is the payment received.}

We suppose that there is a shared budget $B$ across tasks. 
Since $B$ is unlikely to be sufficient to incentivize all users to perform all tasks, we seek to distribute the budget so as to ensure fair training across the tasks. We consider two types of ``fairness:'' \textit{budget-fairness}, i.e., an equal budget for each task; and \textit{max-min fairness}, distributing the budget so as to maximize the minimum take-up rate across tasks. Budget fairness seeks to equalize the resources (in terms of budget) across task, while max-min fairness seeks to make the training outcomes fair.

\subsection{Fairness With Respect to the Same Budget across Tasks}\label{sec:budget-fair}
\textcolor{blue}{
Here, fairness is with respect to 
an equal budget $B$ across tasks. 
We aim to maximize the total take-up rate across tasks:} 
\begin{align}\label{eq:budget_fair}
    & \max_{\textbf{x}_s\in\{0,1\}, \textcolor{magenta}{\textbf{P}}}\ \sum_{s \in S} \left( \sum_{i\in I} x_{i,s} \right) \\ 
    & \text{s.t.}\  \sum_{i \in I} p_{i,s} x_{i,s} \leq \frac{B}{S}, \forall s, \\
    & \,
    p_{i,s} \geq b_{i,s}  \mathbbm{1}_{\{x_{i,s} \geq 0\}},\  \forall i.
\end{align}
\textcolor{black}{The decision variables for the server/ training coordinator are $\mathbf{x}_s$, a binary vector indicating whether each user is incentivized for task $s$, and \textcolor{magenta}{$\mathbf{P}=\{..., p_{i,s}, ... \}$}, the set of payments to each user for each task. 
The first constraint is the budget constraint for each task, and the second is that the user $i$ agrees to train task $s$ if $u_{i,s} \geq 0$.}

We can decouple (\ref{eq:budget_fair}) across tasks and thus consider each task separately. 
\textcolor{blue}{The Proportional Share Mechanism proposed in~\cite{singer2014budget} solves this problem by designing a mechanism that incentivizes users to report truthful bids $b_{i,s} = c_{i,s}$: after ordering bids in ascending order, for $i = 1,2,\ldots$, check if $b_{i,s} \lessgtr \frac{B}{S i}$.
Find the smallest $k$ s.t. $b_{k,s} > \frac{B}{Sk}$. This bid $b_{k,s}$ is the first loser, not in the winning set. All bids smaller than $b_{k,s}$ will be in the winning set and are paid $\frac{B}{S(k-1)}$.} 

While this \underline{budget-fair auction} is truthful, 
tasks disfavored by users (i.e., for which all users have high $c_{i,s}$) may only be able to incentivize \textcolor{black}{few clients, leading to poor training performance.}
Thus, we next \textcolor{black}{look at max-min fairness.} 



\subsection{Fairness with Respect to Max-Min Task Take-Up Rates}
\label{subsec:MaxMinFairAuction}

We aim to maximize the minimum client take-up rate across tasks $\min_{s \in S}\ (\sum_{i \in I} x_{i,s} )$, in light of the user valuations $c_{i,s}$. 
\begin{align}
    \label{eq:maxmin_auction}
    &\max_{\textbf{x}_s\in\{0,1\}, \textcolor{magenta}{\textbf{P}}} \min_{s \in S}\ (\sum_{i \in I} x_{i,s} )\ \\
    &\text{s.t.}\  \sum_{s \in S} \sum_{i \in I} p_{i,s} x_{i,s} \leq B, \, \\   
     & p_{i,s} \geq b_{i,s}
\end{align}
As in (\ref{eq:budget_fair})'s formulation, the decision variables are $\mathbf{x}_s$, and $\textcolor{magenta}{\mathbf{P}}$. The first constraint is the budget constraint, where the budget $B$ is coupled across the tasks $s\in S $. The second constraint is that users' payments must exceed their bids.
We first provide an algorithm to solve (\ref{eq:maxmin_auction}) in Algorithm~\ref{algo:maxmin_greedy}, which loops across tasks and greedily adds a user to each task until the budget $B$ is exhausted. 
\begin{algorithm}[t]
	\caption{GMMFair: Greedy max-min fair algorithm}
 \label{algo:maxmin_greedy}
 \small
	\begin{algorithmic}[1]
        \State \textbf{Initialization:} $S_s = \{\} \, \forall s$, each task $s$ has no users
            \For {task $s = 1,2,\ldots,S$}
                \State Order bids for task $s$ in ascending order
            \EndFor
            \State $t = 1$. Let $\underline{b}_{t,s}$ denote the $t$-th smallest bid \textcolor{orange}{for task $s$, and $\underline{u}_{t,s}$ denote the corresponding user. }
		\While {$\sum_s \underline{b}_{t,s} \leq B$, Check if there is enough budget to add one more user to each task}
                \For {task $s = 1,2,\ldots,S$}
                    \State $S_s \leftarrow S_s \cup \textcolor{orange}{\underline{u}_{t,s}}$; add user with minimum bid to $S_s$
                    \State $B \leftarrow B -\textcolor{orange}{\underline{b}_{t,s}}$ 
                    update remaining budget
                \EndFor
                \State $t\leftarrow t + 1$
		\EndWhile
	\end{algorithmic} 
\end{algorithm}


\begin{lemma}[Max-Min Fair Solution Algorithm]
    Algorithm~\ref{algo:maxmin_greedy} solves the optimization problem (\ref{eq:maxmin_auction}). 
\label{lemma:algo2optimizesMaxMin}
\end{lemma}
\textit{Proof sketch.}
Let $\underline{b}_{t,s}$ denote the $t$-th smallest bid \textcolor{orange}{for task $s$, and $\underline{u}_{t,s}$ denote the corresponding user.}
For an optimal max-min objective value $t$, we show by contradiction that all users with bids $\leq \textcolor{orange}{\underline{b}_{t,s}}$ must be included in the optimal set. Algorithm~\ref{algo:maxmin_greedy} finds this solution. See the Technical Report~\cite{MMFLappendix} for the full proof. \qed.
While Algorithm~\ref{algo:maxmin_greedy} is optimal, it is not truthful: any user $i$ in the winning set for a task $s$ can easily increase its payment by increasing its bid $b_{i,s} > c_{i,s}$. Thus, we next design an auction method that, like Algorithm~\ref{algo:maxmin_greedy}, proceeds in rounds to add an additional user to each task. We adapt the payment rule from the budget-fair auction in Section~\ref{sec:budget-fair} to help ensure truthfulness.
\textcolor{blue}{Our novel \textbf{max-min fair auction} is presented in \textbf{Algorithm \ref{algo:maxminFair_fractional_algo}}. The main idea, is to begin with a budget-fair auction for each task,}
and allocate users to tasks in rounds until a task runs out of budget (lines 4-9), as in Algorithm~\ref{algo:maxmin_greedy}. Once a task runs out of budget, we then re-allocate the budget across tasks so as to keep allocating users to all tasks, if possible (lines 11-16). 
If there is insufficient budget to do so, we can optionally perform a \textit{fractional} allocation that uses up the remaining budget while remaining fair (lines 17-23). 
This ensures max-min fairness, and that the difference in number of users allocated across tasks is at most a fraction. In the learning process, this fraction can be modelled in terms of the time spent training tasks.
\textcolor{orange}{As mentioned earlier, the auction takes place before training, to avoid additional overhead.}
\begin{algorithm}[t]
	\caption{MMFL Max-min Fair 
 }
 \label{algo:maxminFair_fractional_algo}
	\begin{algorithmic}[1]
 \small
        \State \textbf{Initialization:} $B^0_s \gets \frac{B}{s}$, equal initial budget $\forall$ tasks, \textcolor{orange}{where $B^t_s$ is the budget allocated to task $s$ at time $t$.} 
        \State For all tasks, order received bids $b_i$ in ascending order.
        \For{\textcolor{orange}{round $t$} = 1,2, \ldots} 
    		\For{all tasks $s$}
                \State \textcolor{orange}{
                Let $\underline{b}_{t,s}$ denote the $t$-th} biggest bid, for task $s$, \textcolor{orange}{and $\underline{u}_{t,s}$ denote the corresponding user.}
                \If{$\underline{b}_{t,s} \leq \frac{B}{t}$}
                    \State User joins the winners' set 
                    \State and the Payment is updated $\underline{p}_{t,s} \gets \frac{B_s^t}{t}$.
                \EndIf		
            \EndFor
            \If{ $\exists \geq 1$ task s.t. the $t$-th user does not join}
                \State Let $S_{t-1}$ and $S_t$ denote the set of tasks with  $t-1$ and $t$ users respectively.
                \State $A=\sum_{s \in S_{t-1}} [ \underline{b}_{t,s} t -B_s^{t-1}]$
                \State $C=\sum_{s \in S_t} [B_s^{t-1} - \underline{b}_{t,s} t]$
                \If{$A<C$}
                    \State Re-allocate amount $A$ from $S_t$ to $S_{t-1}$, obtaining $A$ from $ s\in S_t$ via a waterfilling manner.
                \Else
                    \State Fractionally allocate $C$ across $S_{t-1}$.
                        \State $\underline{p}_{t,s} =           
                            \begin{cases}
                                \frac{C}{|S_{t-1}|}\ & \text{if} \frac{C}{|S_{t-1}|} \leq \underline{b}_{t,s} \\
                                \underline{b}_{t,s} & \text{otherwise}\
                            \end{cases}$
                        \State $\underline{x}_{t,s} = 
                            \begin{cases}
                                \frac{C}{|S_{t-1}| \ \underline{b}_{t,s}} & \text{if} \frac{C}{|S_{t-1}|} \leq
                                \underline{b}_{t,s}\\
                                1 & \text{otherwise}
                            \end{cases}$ 
                    \State End the Auction.
                \EndIf
            \EndIf
        \EndFor
	\end{algorithmic} 
\end{algorithm}

Without the re-allocation across tasks, we would expect this auction to be truthful. Indeed, in Theorem \ref{theorem:casesUntruthful}, we show that such re-allocation rounds are the only instances that incentivize un-truthfulness.
\textcolor{orange}{In practice, it would be \textit{difficult for users to exploit these conditions}: 
all four conditions rely on the un-truthful bid comprising part of a re-allocation round. 
This heavily depends on other users' bids, which would not be known. 
Moreover, since re-allocations are more likely to happen near the end of the auction, such un-truthful bids have a high possibility of not being in the winning set at all.}
\begin{theorem}[Untruthfulness Conditions]
\label{theorem:casesUntruthful}
    The truthful strategy dominates, except in the following cases where user $i$ obtains higher utility by bidding $\hat{b}_{i,s} >c_{i,s}$. In all cases, $\nexists\,\underline{b}_{j,s} <\hat{b}_{i,s}$ for which a) $\underline{b}_{j,s} + \sum_{s' \in S_{j-1}} \underline{b}_{j,s'}>C$ and b) $\underline{b}_{j,s} > \frac{B_s^j}{j}$. 
    \begin{itemize}
        \setlength{\itemindent}{1.5em}
        \item[\underline{Case 1}] 
        c) $\hat{b}_{i,s} > \frac{B_s^t}{t}$ and d) $\hat{b}_{i,s} + \sum_{s' \in S_{t-1}} \underline{b}_{t,s'} <C$.  
        \item[\underline{Case 2}] If there is no $b_{j,s}$, $b_{i,s} < b_{j,s} < \hat{b}_{i,s}$, and no $\underline{b}_{j,s}$, $ \underline{b}_{j,s} <b_{i,s}<\hat{b}_{i,s}$ for which 
        e) $b_{i,s} > \frac{B^t_s}{t}$ f) $b_{i,s} + \sum_{s' \in S_{t-1}} \underline{b}_{t,s'} >C$, g) $\hat{b}_{i,s} > \frac{B^t_s}{t}$, h) $\hat{b}_{i,s} + \sum_{s' \in S_{t-1}} \underline{b}_{t,s'} >C$, and i) $ \min(b_{i,s},\frac{C}{|S_{t-1}|})< \min(\hat{b}_{i,s},\frac{C}{|S_{t-1}|})$. 
        \item[\underline{Case 3}] If there is no $\underline{b}_{j,s}$, $b_{i,s} <\underline{b}_{j,s} <\hat{b}_{i,s}$ for which 
        j) $\hat{b}_{i,s} + \sum_{s' \in S_{t-1}} \underline{b}_{t,s'} >C$: k) (there exists $l <t$ s.t. $b_{i,s} < \frac{B^l_s}{l}$), l) $\hat{b}_{i,s} > \frac{B_s^t}{t}$, and m) $\hat{b}_{i,s} < \frac{C}{|S_{t-1}|}$ or $\frac{C}{|S_{t-1}|}> \frac{B^l_s}{l}$. 
        \item[\underline{Case 4}] If there is no $b_{j,s}$, $b_{i,s} < b_{j,s} < \hat{b}_{i,s}$, and no $\underline{b}_{j,s}$, $ \underline{b}_{j,s} <b_{i,s}<\hat{b}_{i,s}$ for which 
        n) $b_{i,s} > \frac{B^t_s}{t}$ o) $b_{i,s} + \sum_{s' \in S_{t-1}} \underline{b}_{t,s'} <C$, p) $\hat{b}_{i,s} > \frac{B^t_s}{t}$, and q) $\hat{b}_{i,s} + \sum_{s' \in S_{t-1}} \underline{b}_{t,s'} >C$, r) $b_{i,s} < \min( \frac{C}{|S_{t-1}|}, \hat{b}_{i,s})$.
    \end{itemize}
\end{theorem}
\begin{proof}
We characterize the untruthful cases (all during re-allocation rounds). The negation of the conditions a) and b) indicate that the auction does not end at other bids $\underline{b}_{j,s}$ smaller than the untruthful bid $\hat{b}_{i,s}$. 

\underline{Case 1:} c) re-allocation required for the un-truthful bid and d) there is sufficient budget in $S_t$ for re-allocation.
Note that the truthful bid either entered in an earlier round, or also received re-allocation.

\underline{Case 2:} For both the truthful and un-truthful bid, budget re-allocation across task needs to happen (e and g), but there are insufficient funds (f and h). Fractional allocation happens, and this un-truthful strategy dominates when  $\min(b_{i,s},\frac{C}{|S_{t-1}|})< \min(\hat{b}_{i,s},\frac{C}{|S_{t-1}|})$ (i).

\underline{Case 3:} 
The truthful bid enters the auction at an earlier round $l$ (k), there is insufficient budget for the untruthful bid (l), and insufficient budget for re-allocation (j). Fractional allocation happens, and this untruthful strategy dominates when either $\hat{b}_{i,s} < \frac{C}{|S_{t-1}|}$ or $\frac{C}{|S_{t-1}|}> \frac{B^l_s}{l}$ (m). 

\underline{Case 4:} Re-allocation is required for both the truthful and untruthful bid (n and p). There is sufficient budget for re-allocation for the truthful bid (o), unlike the untruthful bid (q). This untruthful strategy dominates if $b_{i,s} < \min( \frac{C}{|S_{t-1}|}, \hat{b}_{i,s})$.
\end{proof}

Next, we show that even if a user successfully submits an un-truthful bid, it would have limited impact on the max-min fairness objective: 
\begin{corollary}[Untruthfulness Impact]
\label{corollary:UntrImpactMaxmin}
    Suppose a user submits an un-truthful bid under one of Theorem~\ref{theorem:casesUntruthful}'s conditions. Under the un-truthful setting, the fractional re-allocation round will arrive at most one round earlier,
    compared to the truthful setting. This reduces the max-min objective by at most 2. 
\end{corollary}
\begin{proof}
    The proof follows from analyzing each case in Theorem~\ref{theorem:casesUntruthful}.

    Case 1b: Case 1 can be divided into two cases: a) the truthful bid ($T$) entering in an earlier round, and b) the truthful bid receiving re-allocation. For both cases what is common is that the untruthful bid ($U$) receives re-allocation.
    For case b) let $a_1$ and $a_2$ be bids greater than $U$ for task $s$ ($a_1$ need not be the next greatest bid after $U$, but $a_2$ is the next greatest bid after $a_1$). Proving by contradiction we assume that both $a_1, a_2$ are winners under the truthful setting and non-winners in the non-truthful setting.
    The situation $T+a_1 <b_r, U+a_1 > b_r$ is possible, causing the fractional round to come one round earlier when the un-truthful bid is made. (Where $b_r$ is the remaining budget after taking out the budget re-allocated to other bids in other tasks, and re-allocated to other bids in task $s$ between $U$ and $a_1$.)
    The situation $T+a_1 +a_2 <b_r, U+a_1 > b_r$ is not possible, as $a_2$ is greater than $U$. Hence the fractional round comes \textbf{at most one round} earlier when the un-truthful bid is made.
    
    Case 4: The same reasoning as that above works for Case 4, where 
    $a_1$ in the untruthful case is replaced by $U$, as $U$ is the bid in the fractional allocation round. WLOG $U$ comes directly after $o_3$. 
    We have [$T+o_1+o_2+o_3<b_r, o_1+o_2+o_3+U>b_r$] possible but [$T+o_1+o_2+o_3+a_2<b_r, o_1+o_2+o_3+U>b_r$] impossible as $U<a_2$.
    Thus the gap in max-min objective in the truthful vs untruthful setting is at most one.

    Case 1a: We let $o_1, o_2$, etc.. be the bids next in order after $T$. Supposing for the truthful setting we have $[ T<\frac{B}{t}, o_1< \frac{B}{t+1}, o_2 < \frac{B}{t+2} ]$ and for the untruthful setting we have $[o_1 <\frac{B}{t}, o_2 < \frac{B}{t+1}, o_3 \nless \frac{B}{t+2}] $ (where $U$ is larger than $o_3$) - this means that in the untruthful setting $o_3$ is the first round where re-allocation occurs. Going back to the truthful setting, we will be unable to have $"o_3 < \frac{B}{t+3}"$, because $\frac{B}{t+3}< \frac{B}{t+2}$ and we already have $o_3 \nless \frac{B}{t+2}$ from the untruthful setting. Hence the re-allocation round occurs at most one round earlier under the un-truthful setting, given that all other bids are fixed.
    Next, to show that the fractional round comes \textbf{at most one round earlier}, we use a similar reasoning as the above cases: $[o_3+o_4+a_1<b_r, o_3+o_4+U+a_1>b_r]$ is possible, while $[o_3+o_4+a_1+a_2<b_r, o_3+o_4+U+a_1>b_r]$ is not possible since $a_2>U$.

    Case 3: The link between Case 3 and Case 1a, is analogous to the link between Case 4 and Case 1b. Hence, a similar reasoning to that of 1a works for Case 3, 
    where $a_1$ in the untruthful case will be replaced by $U$, as $U$ is the bid in the fractional allocation round. WLOG $U$ comes directly after $o_5$. We thus have $[o_3+o_4+o_5<b_r, o_3+o_4+o_5+U>b_r]$ being possible while $[o_3+o_4+o_5+a_2<b_r, o_3+o_4+o_5+U>b_r]$ being impossible as $U<a_2$. This means that the fractional round comes at most one round earlier, in the untruthful setting.
    
    Case 2: For case 2, both the truthful and un-truthful bids lead to a fractional re-allocation at the $t$th round, after which the auction ends.  
    Fractional allocation happens at the same round for both, because there is only one round of fractional allocation under our auction, and Otherwise, a bid between $T$ and $U$ would have entered, and $U$ will not be in the winning set.
    Thus the gap in max-min objective in the truthful vs untruthful setting is at most one.

    When the fractional round comes at most one round earlier (as in some of the above cases), difference in max-min objective can be bounded by 2.

    
    
\end{proof}

We now compare our max-min fair auction to the (truthful) budget-fair auction in Section \ref{sec:budget-fair}. Intuitively, we expect that it may lead to more fair outcomes. We can quantify this effect by \textcolor{orange}{comparing the outcomes of the two auctions, for} the extreme case in which one task has no users, even if user bids from two tasks follow the same distribution:

\begin{corollary}[Auction Comparison]\label{cor:auctions_compare}
Suppose that all user bids for a given task $s$ are independent samples from a probability distribution $f_s$. Let $x_s\sim \bar{f}_{s}$ denote the minimum of the $K$ user bids for task $s$. Then the max-min fair auction has a smaller probability that at least one task is allocated no users compared to the budget-fair auction: \begin{equation*}
   \mathbb{P}\left[\sum_{s = 1}^{S} x_s \geq B\right] \leq \mathbb{P}\left[\max_{s = 1}^{S} x_s \geq \frac{B}{S}\right].
\end{equation*}
For example, if $f_s\sim \exp(\lambda)$ and $S = 2$, then the probability at least one task is allocated no users in the max-min fair auction is $\exp(-B\lambda)(1 + B\lambda)$, which is strictly greater than the probability at least one task is allocated no users in the budget-fair auction, $\exp(-B\lambda)\left(2\exp\left(\frac{B\lambda}{2}\right) - 1\right)$ if $\lambda B > 0$.
\end{corollary}
\begin{proof}
    Under the max-min fair auction, a given task will have no users allocated if the sum of the lowest bids for each task is larger than the budget $B$. Similarly, under the budget-fair auction, a given task will have no users allocated if its lowest bid exceeds its budget $B/S$. 
    When $f_s\sim \exp(\lambda)$ and $S = 2$, $x_1 + x_2\sim\gamma(2,1/\lambda)$, and $\max(x_1, x_2)\sim(1 - \exp(-\lambda x))^2$. We can then show
   $ \left(1 - \sum_{s = 0}^1 (B\lambda)^s \exp(-B\lambda)\right) \geq \left(1 - \exp\left(\frac{-B\lambda}{2}\right)\right)^2$.
\end{proof}
    %
\subsection{Convergence Guarantees \textcolor{brown}{of our Incentive-Aware Auctions}}
We provide another convergence bound, based on the results of our incentive-aware auctions. For this, we firstly characterise the probabilities of users being in the auctions' winning set (i.e., probabilities they agree to train a particular task). 

Under the \underline{Budget-Fair Auction}, the probability $p_{B,b_{i,s}}^{join}$ that a user with bid $b_i$ joins task $s$, given a total budget of $B$ is 
\begin{equation*}  
\sum_{\overline{t}} P(b_{i,s}\ \text{is}\ \overline{t}\text{-th smallest}) P(\frac{B}{s\overline{t}}>b_{i,s}).
\end{equation*}

Under the \underline{Max-min Fair Auction}, the probability $p^{join}_{B,b_{i,s}}$ that a user with bid $b_i$ joins task $s$, given a total budget of $B$ 
is
\begin{align*}  
& \sum_{\overline{t}} P(b_{i,s} \text{is}\ \overline{t}\text{-th smallest)}P(\frac{B}{s \overline{t}} >b_{i,s})\\
+ &\sum_{\hat{t}} P(b_{i,s} \text{is}\ \hat{t}\text{-th smallest)} P(\frac{B}{s \hat{t}}<b_{i,s})
P(\sum_{s \in S_{\hat{t}-1}} \underline{b}_{\hat{t},s}<C )
\end{align*}
for full participation, and 
\begin{align*}
&\sum_{\overline{t}} P(b_{i,s}\ \text{is}\ \overline{t}\text{-th smallest}) P(\frac{B}{s \overline{t}} < b_{i,s})\\
\times &P(\sum_{s \in S_{\overline{t}-1}} \underline{b}_{\overline{t},s}>C ) \min ( 1, \frac{C}{|S_{t-1}|b_{i,s}}) 
\end{align*}
for partial participation.
\textcolor{black}{Letting $\mathbb{W}$ be the set of auction winners, and $W_K$ be all the subsets of $K$ clients recruited from $N$ users, we present a new convergence bound which takes into account the joint probability of users' being in the auction's winners set.
} 

\begin{theorem}
    [Convergence with Auctions] With Assumptions 1-4, under a decaying learning rate $\eta_t = \frac{1}{\mu (t+\gamma)}$, under an auction scheme and our difficulty-aware client-task allocation algorithm, the error after $T$ iterations of federated learning for a specific task $s$ in MMFL, $\mathbb{E}[f_s(\overline{\textbf{w}}^{(T)}_{s})]-f_s^*$, is upper-bounded by
    \begin{align*}
        & \frac{1}{T+\gamma } \bigg[ [\textcolor{black}{64} \tau^2 G^2 \textcolor{black}{\sum_{K=1}^N \sum_{\mathbb{W} \in W_K} J_{\mathbb{W}}} + 4\sigma^2]/ (3 \overline{\rho}^{\alpha}\mu^2) \\ 
        & + \frac{8 L^2 \Gamma_s}{\mu^2} +\frac{L \gamma \| \overline{\textbf{w}}^{(0)}_{s} -\textbf{w}^*_{s} \|^2 }{2} \bigg] + \frac{8 L \Gamma_s}{3 \mu} (\frac{\tilde{\rho}^{\alpha}}{\overline{\rho}^{\alpha}}-1)
    \end{align*}
where 
$\textcolor{black}{J_{\mathbb{W}}= \Pi_{i \in \mathbb{W}} ( p^{join}_{B, b_i}) \Pi_{j \notin \mathbb{W}} (1- p^{join}_{B, b_j})}.$
Note that $p^{join}_{B,b_i}$ will be modified depending on the auction used. 
\label{theorem:AuctionConvg}
\end{theorem}
\textit{Proof Sketch.}
We use a Poisson binomial distribution to model users being in the winning set of the auction. We then follow the proof of Corollary~\ref{corollary:firstConvgBound}.
\qed


\section{Evaluation}
\label{sec:evaluation}

\begin{figure}[t]
  \centering
  \subfigure[Minimum accuracy across the 3 tasks.]{\includegraphics[scale=0.35]{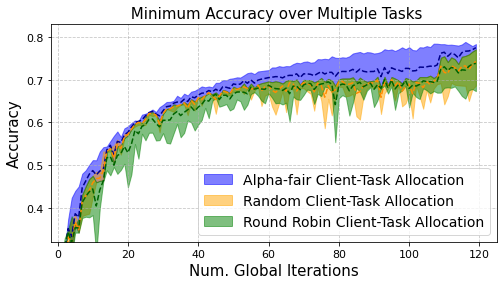}}
   \subfigure[Accuracy for task 3: Fashion-MNIST.]
  {\includegraphics[scale=0.35]{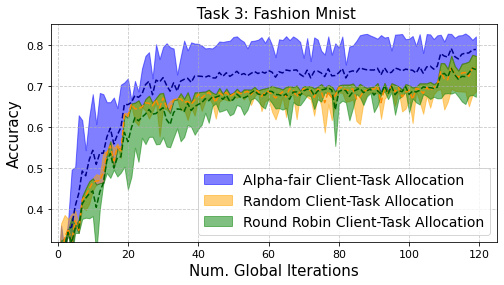}}
  \subfigure[\textcolor{orange}{Accuracy for Task 1: MNIST.}]
  {\includegraphics[scale=0.43]{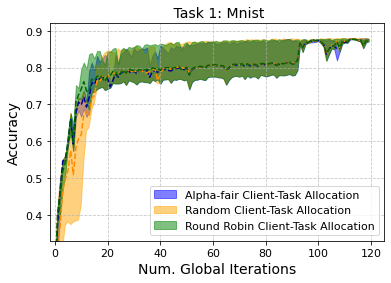}}
  \subfigure[\textcolor{orange}{Accuracy for Task 2: CIFAR-10.}]
  {\includegraphics[scale=0.43]{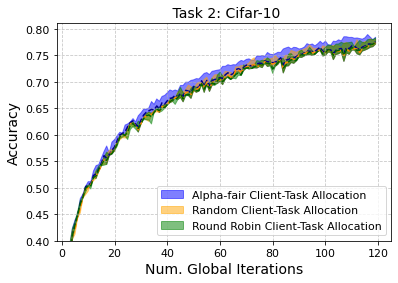}}
  \caption[Convergence.]
  {\textcolor{brown}{\textbf{Results of Experiment 1: 3 tasks with varying difficulty levels.} 
  The 3 tasks have datasets MNIST, CIFAR-10, and Fashion-MNIST. 
  \textbf{FedFairMMFL} achieves a higher minimum accuracy across tasks, as compared to the baselines, while maintaining the same or higher accuracy for the other 2 tasks.}
  }
  \vspace{-4mm}
  \label{fig:mcf_120users}
\end{figure}


\begin{figure*}[t]
  \centering
\subfigure[Variance, Num. tasks=4.]{\includegraphics[scale=0.32]{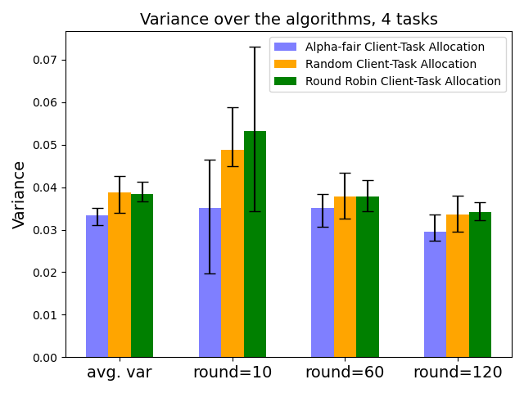}}
  \subfigure[Variance, Num. tasks=5.]{\includegraphics[scale=0.32]{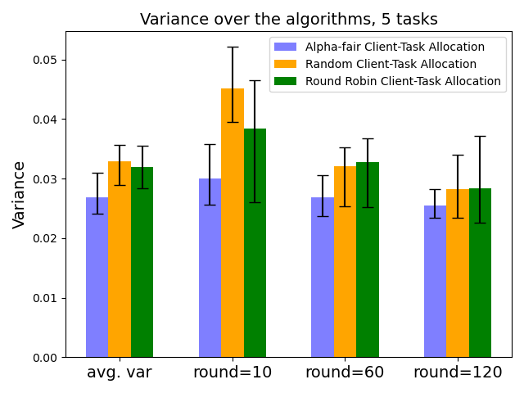}}
  \subfigure[Variance, Num. tasks=6.]{\includegraphics[scale=0.32]{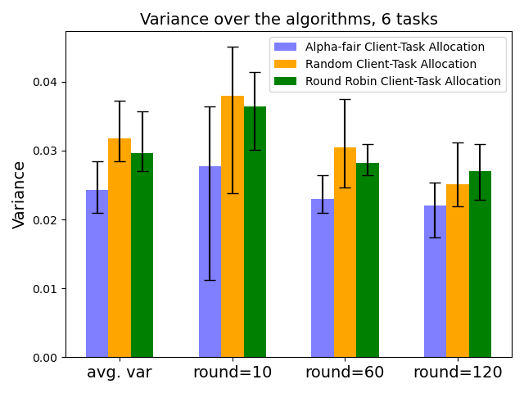}}
  \subfigure[Variance, Num. tasks=10.]{\includegraphics[scale=0.32]{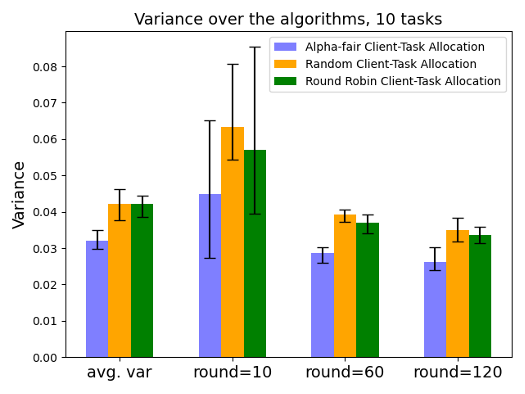}}
  \\
  \subfigure[Min accuracy, Num. tasks=3.]{\includegraphics[scale=0.35]{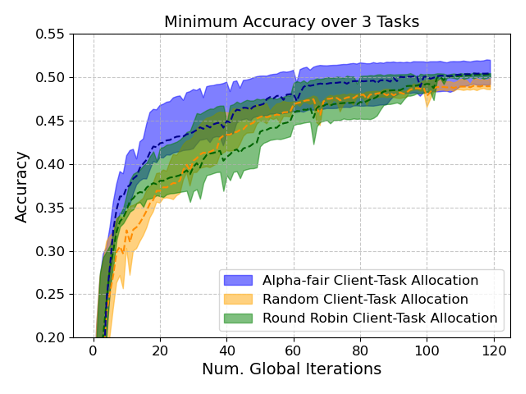}}
  \subfigure[Min accuracy, Num. tasks=6.]{\includegraphics[scale=0.35]{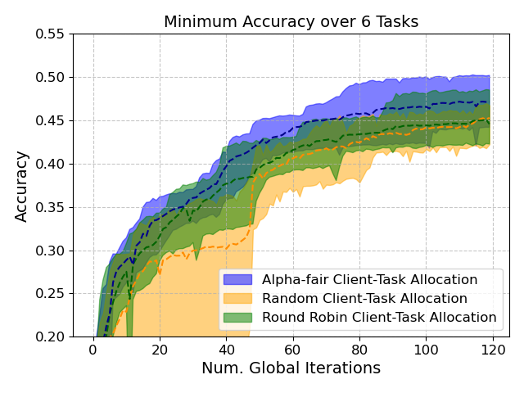}}
  \subfigure[Min accuracy, Num. tasks=10.]{\includegraphics[scale=0.35]{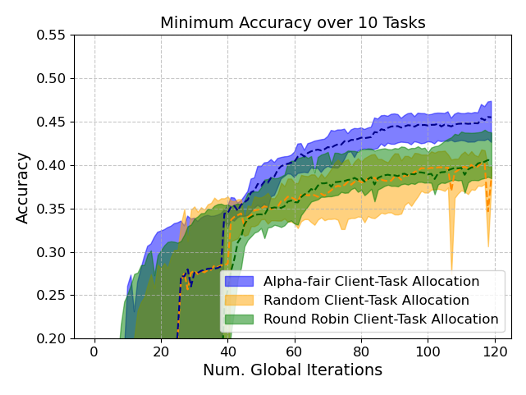}}
  \caption[Convergence.]
  {\textcolor{teal}{\textbf{Results of Experiment 2: Increasing the number of tasks.} 
  (a), (b), (c), and (d) compare the variance in accuracy observed across multiple tasks, when employing the different algorithms. 
  The error bars denote the max/min values of the corresponding variance among all 4 random seeds. Our algorithm generally achieves a lower variance amongst task performance. (e), (f), and (g) show the trend of minimum accuracy when task numbers are 3, 6, and 10, respectively. 
  }}
  \label{fig:exp2}
  \vspace{-3.5mm}
\end{figure*}

\begin{figure}[h]
  \centering
  \subfigure[Minimum accuracy over 3 tasks, MNIST, CIFAR-10, and Fashion-MNIST. 
  ]{\includegraphics[scale=0.3]{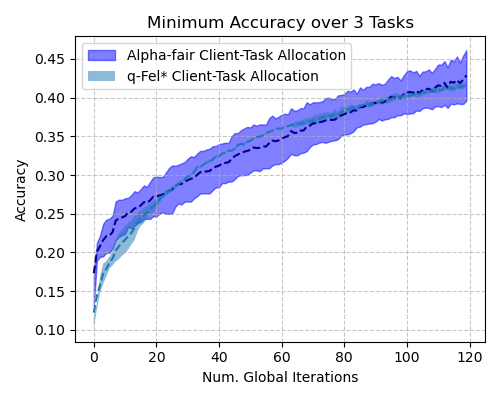}}
  \subfigure[Minimum accuracy over 4 tasks, MNIST, CIFAR-10, EMNIST, and Fashion-MNIST.]{\includegraphics[scale=0.3]{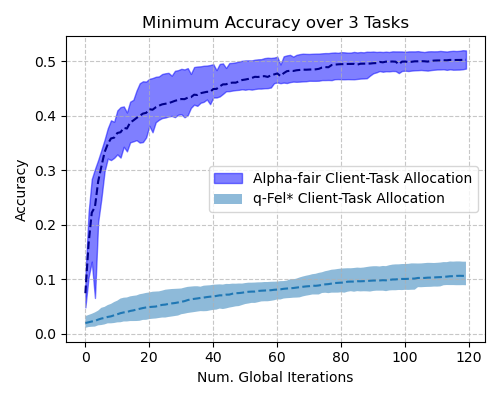}}
  \caption[Convergence.]
  {\textcolor{green}{\textbf{Comparison of two alpha-fair based methods.} 
  a) 3 tasks include MNIST, CIFAR-10, and Fashion-MNIST. 
 b) 4 tasks include MNIST, CIFAR-10, EMNIST, and Fashion-MNIST. 
EMNIST convergence is extremely slow with q-Fel. 
  \textbf{FedFairMMFL} converges faster and achieves a higher minimum accuracy across tasks compared to q-Fel for task fairness. }
  }
  \vspace{-4mm}
  \label{fig:q-fel}
\end{figure}

\begin{figure*}[t]
  \centering
\subfigure[Num. Clients=80.]{\includegraphics[scale=0.3]{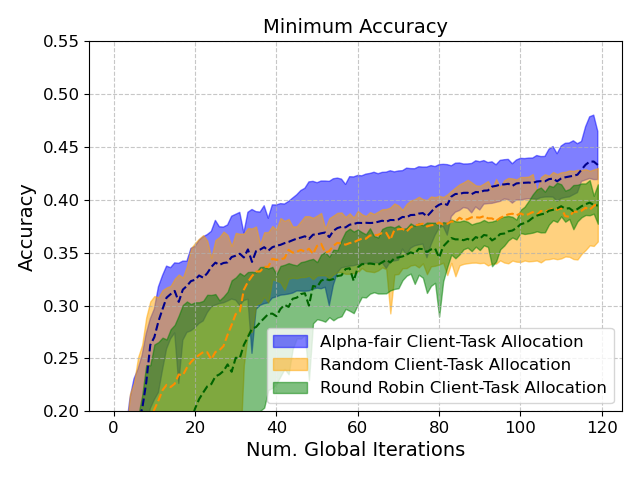}}
  \subfigure[Num. Clients=120.]{\includegraphics[scale=0.3]{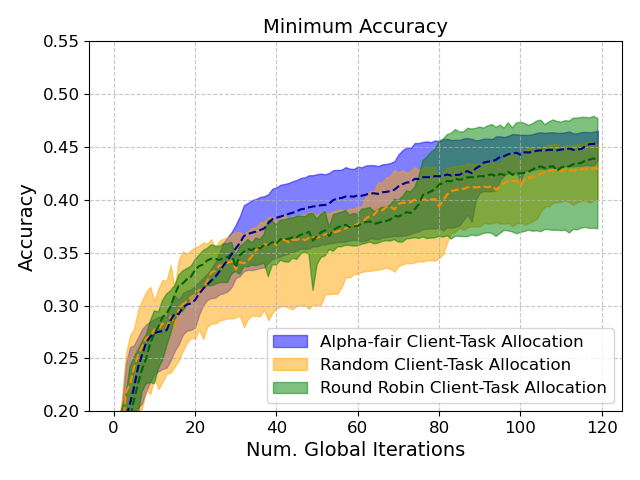}}
  \subfigure[Num. Clients=160.]{\includegraphics[scale=0.3]{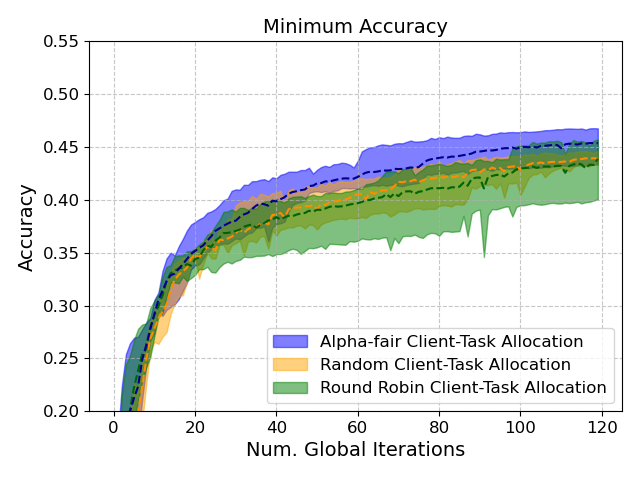}}
  \caption[Convergence.]
  {\textcolor{teal}{\textbf{Results of Experiment 3: Increasing the number of clients.} All experiments have 5 tasks (MNIST, CIFAR-10, Fashion-MNIST, EMNIST, and CIFAR-10). Experiments are conducted with 4 random seeds. The participation rate is 0.25, and client numbers are 80, 120, and 160, respectively. Our algorithm \textbf{FedFairMMFL} generally achieves a higher minimum accuracy, and/or a faster convergence.
  }}
  \label{fig:exp3}
  \vspace{-3.5mm}
\end{figure*}

\begin{figure*}[h]
  \centering
  \subfigure[Random allocation.]{\includegraphics[scale=0.36]{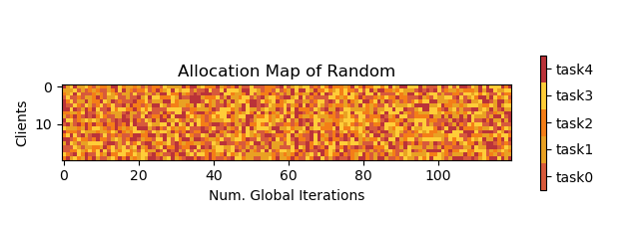}}
  \subfigure[Round robin allocation.]{\includegraphics[scale=0.36]{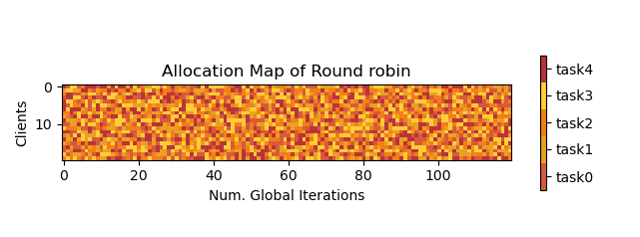}}
  \subfigure[\haoran{FedFairMMFL} ($\alpha=1$).]{\includegraphics[scale=0.36]{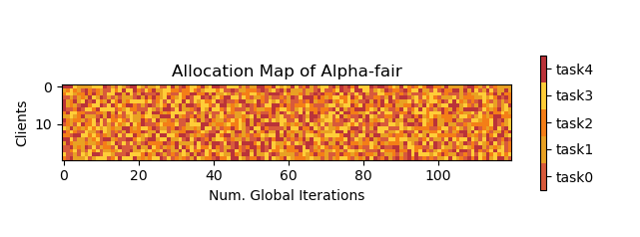}}
  \subfigure[\haoran{FedFairMMFL} ($\alpha=2$).]{\includegraphics[scale=0.36]{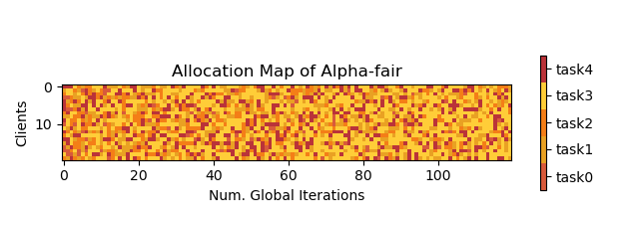}}
  \subfigure[\haoran{FedFairMMFL} ($\alpha=3$).]{\includegraphics[scale=0.36]{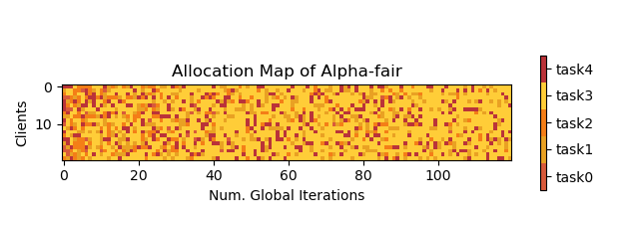}}
  \subfigure[\haoran{FedFairMMFL} ($\alpha=4$).]{\includegraphics[scale=0.36]{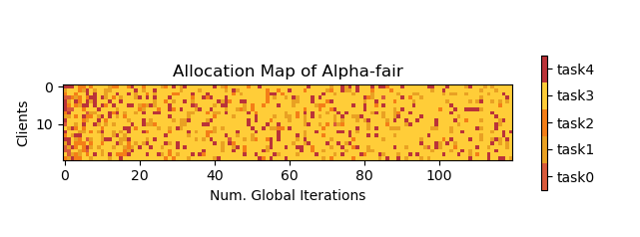}}
  \subfigure[\haoran{FedFairMMFL} ($\alpha=5$).]{\includegraphics[scale=0.36]{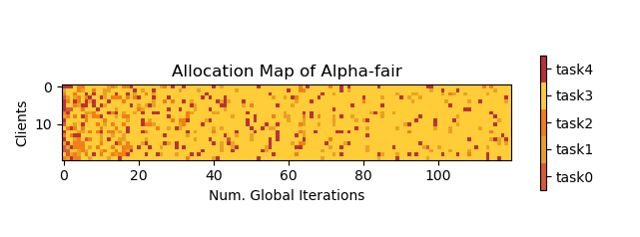}}
  \subfigure[\haoran{FedFairMMFL} ($\alpha=6$).]{\includegraphics[scale=0.36]{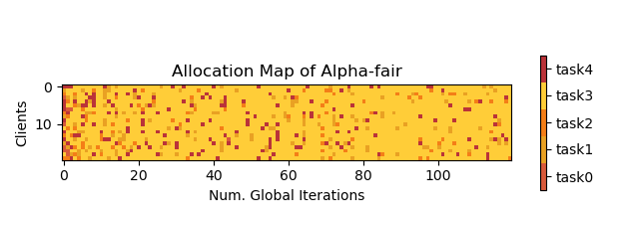}}
  \caption[Convergence.]
  {\textbf{Experiment 4 on different values of $\alpha$}: (a)-(h) show the maps of clients allocation to 5 tasks for 120 rounds of training in two baseline algorithms and FedFairMMFL with different $\alpha$. (c) $\alpha=1$ yields a similar result as the (a) random allocation, and as $\alpha$ increases, more resources are allocated to the hardest task (task 3). \textcolor{pink}{The 5 tasks are MNIST, CIFAR-10, Fashion-MNIST, EMNIST, and CIFAR-10, respectively.} 
  }
  \label{fig:different_alpha}
\end{figure*}

\begin{figure*}[t]
  \centering
\subfigure[Difference in the number of users which will train the tasks in MMFL.]{\includegraphics[scale=0.17]{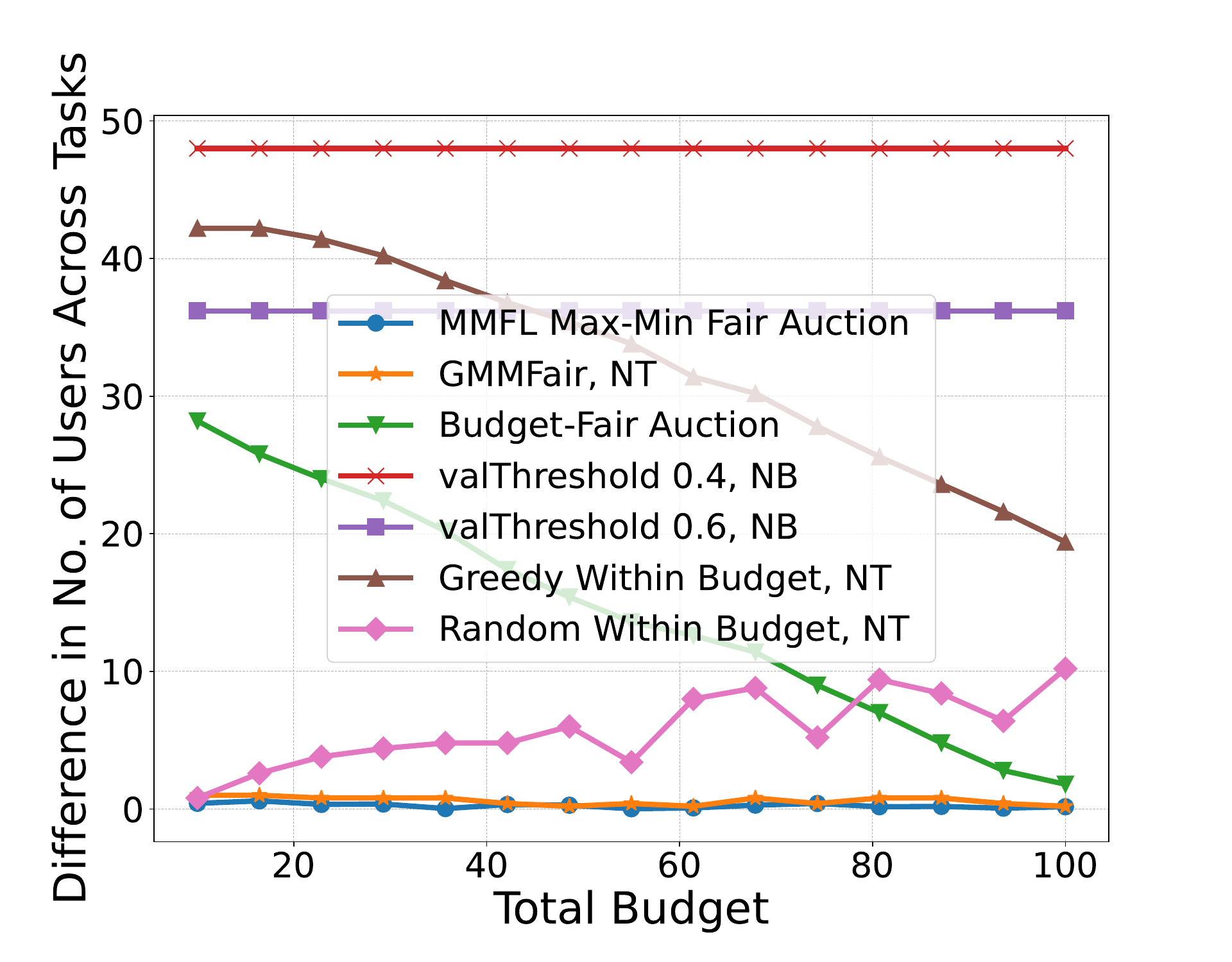}}
  \subfigure[Minimum number of users recruited for training, over two tasks.]{\includegraphics[scale=0.17]{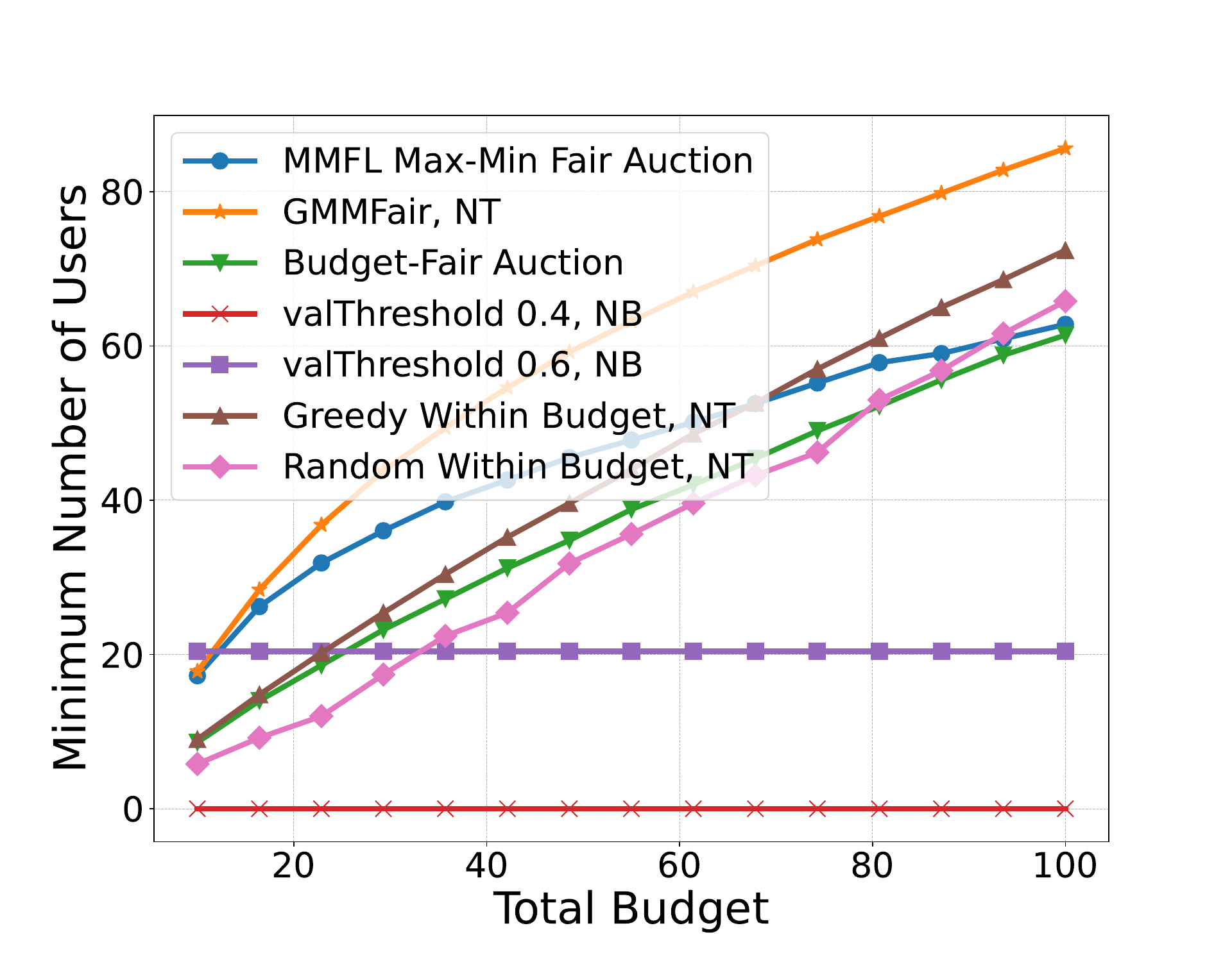}}
   \subfigure[ The minimum accuracy after combining our auction strategies and baselines, with \textbf{FedFairMMFL},  \textcolor{magenta}{i.e. combining the client recruitment part and model training parts of the MMFL pipeline.}]{\includegraphics[scale=0.41]{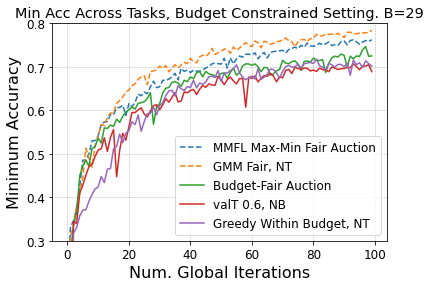}}
  
  \caption[Convergence.]
  {\textbf{Results of Experiment 5: Client Take-Up Rates.} 
  \textcolor{violet}{(NB= No Budget, NT= Non Truthful Mechanism). Our near-truthful auction MMFL Max-Min Fair (a) leads to a lower difference in the number of users across tasks, and (b) incentivizes more clients than any incentivization algorithm but the un-truthful GMMFair. 
  \textbf{Results of Experiments 6: How auctions impact learning.} Fig. (c) shows the minimum accuracy after combining the results of our auction strategies and baselines, with \textbf{FedFairMMFL}. The MMFL Max-Min Fair Auction leads to higher accuracies.}
  }
  \vspace{-5mm}
  \label{fig:ClientRecruitmentStrategiesEval}
\end{figure*}


Finally we evaluate \textbf{FedFairMMFL} and our auction schemes. 
Our \textbf{main results} are that (i) \textit{\textbf{FedFairMMFL} consistently performs fairly} across different experiment settings, \textcolor{orange}{increasing the minimum accuracy and decreasing the variance over tasks, while \emph{maintaining the average accuracy over tasks}}; 
and (ii) \textit{our max-min auction increases the minimum learning accuracy} across tasks in the budget constrained setting. 
We concurrently train 
\textcolor{teal}{(i) three tasks of \textit{varied} difficulty levels, and evaluate the algorithms' performance across varied numbers of (ii) tasks and (iii) clients. We analyze (iv) the impact of our auctions and baselines on the take-up rate over tasks, 
and analyze (v) the impact of these take-up rates on the tasks' learning outcomes. We use $\alpha=3$ for \textbf{FedFairMMFL} in all experiments. The effect of different $\alpha$ is also measured and discussed. 
} 


\textcolor{blue}{Our client-task allocation algorithm \textbf{FedFairMMFL} is measured against \textcolor{green}{three} \textbf{baselines}: a) \textit{Random}, where clients are randomly allocated a task with equal probability each global epoch \textcolor{purple}{(note that this corresponds to taking $\alpha = 1$ in our $\alpha$-fair allocation)}, b) \textit{Round Robin} \cite{bhuyan2022multi}, where active clients 
are allocated to tasks in a round-robin manner, sequentially across global epochs, 
\textcolor{green}{and c) \textit{q-Fel} \cite{li2019fair}, \textcolor{pink}{which also uses alpha-fairness to enhance fairness (but across clients and not tasks), where the parameter $q$ is the $\alpha$-equivalent.} 
Instead of changing the sampling probability as in our \textbf{FedFairMMFL}, \textit{q-Fel} directly changes the aggregation coefficients based on the $\alpha$-fairness function and adjusts the global learning rate accordingly.}
\textcolor{pink}{For the baseline, we use an adapted version of \textit{q-Fel}, where fairness is over tasks and not clients.}
Our main fairness \textbf{metrics} are the minimum test accuracy across all tasks and the variance of the test accuracies (cf. Lemma~\ref{lemma:variance}).
As some models may have larger training losses by nature, 
e.g., if they are defined on different datasets, 
making them more easily favored by the algorithm, we use the test accuracy instead of training loss in \textbf{FedFairMMFL}.
Each experiment's results are averaged over 3-5 random seeds - we plot the average values (dotted line), along with the maximum and minimum values indicated by the shaded area.  
Each client's data 
is sampled from either \textcolor{teal}{all classes} (i.i.d.) or \textcolor{teal}{a randomly chosen half of the classes} (non-i.i.d.).}

\textcolor{violet}{\textbf{Experiment 1: Tasks of Varying Difficulty Levels.}
To model differing task difficulties, 
our tasks are to train image recognition models on MNIST and Fashion-MNIST on a convolutional neural network with 2 convolutional, 2 pooling, and 2 fully connected layers; and CIFAR-10 with a pre-activated ResNet. 
\textcolor{orange}{Our algorithm does not assume that the task with larger model is the more difficult task, rather it takes the current performance (accuracy) levels as an indicator on which task needs more clients currently.}
Experiments are run with 120 clients 
having $150-250$ datapoints from five (randomly chosen) of the ten classes (i.e., non i.i.d. data); $35\%$ of clients, chosen uniformly at random, are active each round. 
As seen in Fig. \ref{fig:mcf_120users} a), \textbf{FedFairMMFL} has a 
higher minimum accuracy than the random and round robin baselines. 
\textbf{FedFairMMFL} adaptively allocates more clients to the prevailing worst-performing task, resulting in a higher accuracy for this task (Task 3: Fashion-MNIST), as seen in Fig. \ref{fig:mcf_120users} b), while maintaining the same or a slightly higher accuracy for the other tasks \textcolor{orange}{(see Fig. \ref{fig:mcf_120users} c and d.)}} Though we might guess that CIFAR-10 would be the worst-performing task, as it uses the largest model, \textbf{FedFairMMFL} identifies that Fashion-MNIST is the worst-performing task, and it achieves fair convergence even in this challenging setting where Fashion-MNIST persistently performs worse than the other tasks.

\begin{table}[!htbp]
\caption{\textcolor{orange}{Experiment 2 Quantitative Results. Our algorithm \textbf{FedFairMMFL} improves the minimum accuracy over tasks; at the same time, our average accuracy across tasks, is comparable with the baselines.}}
\label{tab:exp2}
\centering
\begin{tabular}{@{}c|cc|cc@{}}
\toprule
\multirow{2}{*}{Methods} & \multicolumn{2}{c|}{6 tasks} & \multicolumn{2}{c}{10 tasks} \\ \cmidrule(l){2-5} 
                         & Avg Acc   & Min Acc          & Avg Acc   & Min Acc          \\ \midrule
Random                   & 0.789     & 0.450            & 0.748     & 0.386            \\
Round Robin              & 0.782     & 0.453            & 0.758     & 0.406            \\
FedFairMMFL              & 0.783     & \textbf{0.475}   & 0.754     & \textbf{0.452}   \\ \bottomrule
\end{tabular}
\end{table}

\textcolor{teal}{
\textbf{Experiment 2: Increasing the number of tasks. 
} To further understand the various algorithms (\textbf{FedFairMMFL} and the random allocation and round robin baselines) under different settings, their performance is measured as we increase the number of tasks, from 3 to 4, 5, 6, and 10. The tasks include MNIST, Fashion-MNIST, CIFAR-10, and EMNIST. 
We use the same networks as in Experiment 1 and a similar convolutional neural network with 2 convolutional, 2 pooling, and 2 fully connected layers for the EMNIST task. 
Experiments are run with 20 clients, having a participation rate of 1. 
Each client has $400-600$ datapoints from each dataset for each task. 
We plot the variance amongst tasks on average, as well as the variance at global epoch round numbers 10, 60, and 120.
As seen in Fig. \ref{fig:exp2}, with the increase in the round number, all algorithms' variance decreases, and 
\textbf{FedFairMMFL} has the lowest variance for all settings. This can be explained through the fact that \textbf{FedFairMMFL} dynamically allocates clients with a higher probability to tasks which are currently performing worse, 
eliminating the difference among tasks as training progresses. \textcolor{purple}{This effect remains even as we increase the number of tasks, which slows the training for all tasks as the clients are being shared among more tasks} - see Fig. \ref{fig:exp2} (e), (f), and (g), which compare the minimum accuracy among different algorithms when task numbers are 3, 6, and 10. \textbf{FedFairMMFL} converges faster than both baselines. 
}
\textcolor{green}{
Table \ref{tab:exp2} shows the final accuracies for these algorithms. In terms of task fairness, \textbf{FedFairMMFL} achieves higher minimum accuracy across tasks. Besides, even though \textbf{FedFairMMFL} allocates more clients to tasks with the prevailing higher loss values during training, \emph{the average accuracy is not impacted much}, and is comparable with the baselines, as seen in Table \ref{tab:exp2}.}

\textcolor{green}{
Fig. \ref{fig:q-fel} compares two alpha-fair based methods: our algorithm FedFairMMFL, and \textcolor{pink}{an adapted version of q-Fel \cite{li2019fair}, where fairness is over tasks and not clients.} 
We notice that q-Fel converges quite slowly when we include EMNIST in the set of tasks. 
This is because, \textcolor{pink}{to guarantee convergence, q-Fel adjusts the aggregation coefficients and learning rate.} 
\textcolor{pink}{For specific tasks, the adjusted global learning rate (which is inferred from the
upper bound of the local Lipschitz constants of the gradients) can be extremely small, therefore slowing down convergence,} which \cite{li2019fair} also showed in their results. \textbf{FedFairMMFL} directly modifies the client sampling probability and does not modify the aggregation coefficients, providing a more reliable performance in diverse task settings. 
}

\textcolor{teal}{
\textbf{Experiment 3: Increasing the number of clients.} 
Different client numbers are also tested to further evaluate \textbf{FedFairMMFL}'s performance. Here, all experiments have 5 tasks (MNIST, CIFAR-10, Fashion-MNIST, EMNIST, and CIFAR-10), with client participation rate $25\%$ and $200-300$ datapoints for each client. We vary the number of clients 
from 80 to 160. Experiment results are shown in Fig. \ref{fig:exp3}. It is clear to see that in this setting, \textbf{FedFairMMFL} outperforms the other two algorithms, either achieving a higher minimum accuracy, and/or a faster convergence.
Its performance is also further improved as the client number increases, and is more stable among different random seeds\textcolor{purple}{, likely because having more clients allows the worse-performing tasks to "catch up" with the others faster under the $\alpha$-fair allocation. }
}

\textbf{Experiment 4: The impact of different values of $\alpha$.}
\textcolor{orange}{To understand the effect of $\alpha$ on client-task allocation over the whole training process, we test the FedFairMMFL with multiple values of $\alpha$. Experimental results are shown in Fig. \ref{fig:different_alpha}. Here all experiments have 5 tasks (MNIST, CIFAR-10, Fashion-MNIST, EMNIST, and CIFAR-10), and we have $20$ clients that participate in each round. 
Each client has $400-600$ datapoints from each dataset for each task. We use the same models as we used in Experiment 2 and 3. In this setting, EMNIST (task 3) is always the hardest task, and therefore gets more clients as $\alpha$ increases. When $\alpha=1$, allocation is uniform as the algorithm degenerates to the random client-task allocation. When $\alpha$ is too high, most of the resources will be allocated to only one task as shown in Fig. \ref{fig:different_alpha} (g)(h). This uneven allocation can lead to other simpler tasks not receiving sufficient training and requiring additional rounds to achieve satisfactory accuracy. Since $\alpha=3$ balances the fairness (in terms of the minimum accuracy and variance) and the training efficiency (in terms of the average accuracy) in a desirable way, we use $\alpha=3$ in all of our other experiments.}

\textcolor{blue}{\textbf{Experiment 5: Client Take-Up Rates.} We compare the number of clients incentivized to train each task under our budget-fair (Section \ref{sec:budget-fair}) and max-min fair (Section \ref{subsec:MaxMinFairAuction}) auctions, as well as the (un-truthful) GMMFair algorithm, which upper-bounds the max-min fair auction result, 
and 4 other baselines: \textit{valThreshold 0.4} and \textit{valThreshold 0.6} assume no budget but that all users with cost (valuation) $c_{i,s} <0.4$ (respectively $c_{i,s} <0.6$) will agree to train the task; these simulate posted-price mechanisms that show the effect of the incentivization budget. \textit{Greedy Within Budget} and \textit{Random Within Budget} evenly divide the budget across tasks. Clients are added by ascending bid (for \textit{Greedy Within Budget}) and randomly (for \textit{Random Within Budget}), until the budget is consumed; comparison to these baselines shows the cost of enforcing truthfulness in the budget-fair auction.
We consider two tasks and 100 clients; client bids are between 0 and 1 and are independently drawn from a truncated Gaussian distribution for Task 1 and an increasing linear distribution for Task 2. Results are averaged across 5 random seeds.
}

\textcolor{violet}{As seen in Fig. \ref{fig:ClientRecruitmentStrategiesEval} a), our auction MMFL Max-min Fair leads to the lowest difference in the number of clients across tasks, managing to hit the theoretical optimal of the non-truthful algorithm GMMFair. 
\textit{Greedy within budget} assigned the most clients to task 1, resulting in a high "difference" as it assigns less clients to task 2 (increasing linear distribution of valuations), nevertheless \textit{Greedy within budget} (along with GMMFair and \textit{Random within budget}) are not truthful. Both \textit{Greedy within budget} (non truthful) and the Budget-Fair Auction become more fair (decreasing difference across tasks) as the budget increases. Nevertheless MMFL Max-Min Fair outperforms them in the budget constrained region.} 
As seen in Fig. \ref{fig:ClientRecruitmentStrategiesEval} b), GMMFair yields the highest minimum number of users assigned to each task, consistent with Lemma \ref{lemma:algo2optimizesMaxMin}. 
When the budget is low (\textit{budget constrained regime}), our near-truthful auction MMFL Max-min Fair outperforms all the other mechanisms in the minimumn number of users. As the budget increases, the greedy-based mechanisms outperform it. Nevertheless, they and GMMFair are not truthful, and can be manipulated if all users jointly submit high bids. 

\textcolor{black}{\textbf{Experiment 6: Take-Up Rates and Learning Outcomes.} We finally consider a constrained budget ($B=29$) and evaluate how the outcome of the client recruitment strategies \textcolor{magenta}{(our auction strategies and baselines)} impact MMFL model training, 
\textcolor{magenta}{i.e. combining the client recruitment and model training parts of the pipeline (Fig. \ref{fig:MMFLdiag}). Specifically, the client recruitment strategies (\textit{MMFL Max-Min Fair}, \textit{GMM Fair}, and baselines \textit{valThreshold}, \textit{Greedy Within Budget}) determines which client has been recruited for which task. These clients will then be involved in MMFL training under our algorithm FedFairMMFL.}
Fig. \ref{fig:ClientRecruitmentStrategiesEval} c) shows the minimum test accuracy across the two tasks (MNIST and CIFAR-10) at each training round, using \textbf{FedFairMMFL} for client-task allocation. }
GMMFair consistently has the highest minimum accuracy across tasks, likely because it recruits the most clients to each task, as shown in Experiment 5 (Fig. \ref{fig:ClientRecruitmentStrategiesEval} b). Our MMFL auction, which additionally has the benefit of near-truthfulness, achieves nearly the same minimum accuracy, showing that it performs well in tandem with \textbf{FedFairMMFL}.

\section{Conclusion and Future Directions}\label{sec:conclusion}
In this paper, we consider the fair training of concurrent MMFL tasks. 
Ours is the first work to both theoretically and experimentally validate algorithms that ensure fair outcomes across different learning tasks, \textcolor{orange}{which are concurrently trained in an FL setting.} We first propose \textbf{FedFairMMFL}, an algorithm to allocate clients to tasks in each training round, \textcolor{orange}{based on prevailing performance levels,} showing that it permits theoretical convergence guarantees. We then propose two auction mechanisms, \textbf{Budget-Fair} and \textbf{MMFL Max-Min Fair}, that seek to incentivize a fair number of clients to commit to training each task. Our experimental results for training different combinations of FL tasks show that \textbf{FedFairMMFL} and \textbf{MMFL Max-Min Fair} can jointly ensure fair convergence rates and final test accuracies across the different tasks. As the first work to consider fairness in the MMFL setting, we expect our work to seed many future directions, including how to balance achieving fairness over tasks with fairness over clients' utilities. Another direction is to investigate fairness when tasks have differing difficulties and clients have heterogeneous, and dynamically available computing resources for training.





%

\bibliographystyle{IEEEtran}

\bibliography{references} 

\begin{thebibliography}{10}
\providecommand{\url}[1]{#1}
\csname url@samestyle\endcsname
\providecommand{\newblock}{\relax}
\providecommand{\bibinfo}[2]{#2}
\providecommand{\BIBentrySTDinterwordspacing}{\spaceskip=0pt\relax}
\providecommand{\BIBentryALTinterwordstretchfactor}{4}
\providecommand{\BIBentryALTinterwordspacing}{\spaceskip=\fontdimen2\font plus
\BIBentryALTinterwordstretchfactor\fontdimen3\font minus \fontdimen4\font\relax}
\providecommand{\BIBforeignlanguage}[2]{{%
\expandafter\ifx\csname l@#1\endcsname\relax
\typeout{** WARNING: IEEEtran.bst: No hyphenation pattern has been}%
\typeout{** loaded for the language `#1'. Using the pattern for}%
\typeout{** the default language instead.}%
\else
\language=\csname l@#1\endcsname
\fi
#2}}
\providecommand{\BIBdecl}{\relax}
\BIBdecl

\bibitem{mcmahan2017communication}
B.~McMahan, E.~Moore, D.~Ramage, S.~Hampson, and B.~A. y~Arcas, ``Communication-efficient learning of deep networks from decentralized data,'' in \emph{Artificial intelligence and statistics}.\hskip 1em plus 0.5em minus 0.4em\relax PMLR, 2017, pp. 1273--1282.

\bibitem{yao2023fedrule}
Y.~Yao, M.~M. Kamani, Z.~Cheng, L.~Chen, C.~Joe-Wong, and T.~Liu, ``Fedrule: Federated rule recommendation system with graph neural networks,'' in \emph{Accepted to ACM/IEEE IoTDI}, 2023.

\bibitem{aivodji2019iotfla}
U.~M. A{\"\i}vodji, S.~Gambs, and A.~Martin, ``Iotfla: A secured and privacy-preserving smart home architecture implementing federated learning,'' in \emph{2019 IEEE security and privacy workshops}.\hskip 1em plus 0.5em minus 0.4em\relax IEEE, 2019, pp. 175--180.

\bibitem{liu2023cache}
Y.~Liu, L.~Su, C.~Joe-Wong, S.~Ioannidis, E.~Yeh, and M.~Siew, ``Cache-enabled federated learning systems,'' in \emph{Proceedings of the Twenty-fourth International Symposium on Theory, Algorithmic Foundations, and Protocol Design for Mobile Networks and Mobile Computing}, 2023, pp. 1--11.

\bibitem{liao2023accelerating}
Y.~Liao, Y.~Xu, H.~Xu, Z.~Yao, L.~Wang, and C.~Qiao, ``Accelerating federated learning with data and model parallelism in edge computing,'' \emph{IEEE/ACM Transactions on Networking}, 2023.

\bibitem{nguyen2022federated}
D.~C. Nguyen, Q.-V. Pham, P.~N. Pathirana, M.~Ding, A.~Seneviratne, Z.~Lin, O.~Dobre, and W.-J. Hwang, ``Federated learning for smart healthcare: A survey,'' \emph{ACM Computing Surveys}, vol.~55, no.~3, pp. 1--37, 2022.

\bibitem{imteaj2019distributed}
A.~Imteaj and M.~H. Amini, ``Distributed sensing using smart end-user devices: Pathway to federated learning for autonomous iot,'' in \emph{2019 International conference on computational science and computational intelligence (CSCI)}.\hskip 1em plus 0.5em minus 0.4em\relax IEEE, 2019, pp. 1156--1161.

\bibitem{liao2023adaptive}
Y.~Liao, Y.~Xu, H.~Xu, L.~Wang, and C.~Qian, ``Adaptive configuration for heterogeneous participants in decentralized federated learning,'' in \emph{IEEE INFOCOM 2023-IEEE Conference on Computer Communications}.\hskip 1em plus 0.5em minus 0.4em\relax IEEE, 2023, pp. 1--10.

\bibitem{liao2024mergesfl}
Y.~Liao, Y.~Xu, H.~Xu, L.~Wang, Z.~Yao, and C.~Qiao, ``Mergesfl: Split federated learning with feature merging and batch size regulation,'' in \emph{2024 IEEE 40th International Conference on Data Engineering (ICDE)}.\hskip 1em plus 0.5em minus 0.4em\relax IEEE, 2024, pp. 2054--2067.

\bibitem{liao2024parallelsfl}
Y.~Liao, Y.~Xu, H.~Xu, Z.~Yao, L.~Huang, and C.~Qiao, ``Parallelsfl: A novel split federated learning framework tackling heterogeneity issues,'' in \emph{Proceedings of the 30th Annual International Conference on Mobile Computing and Networking}, 2024, pp. 845--860.

\bibitem{askin2024fedast}
B.~Askin, P.~Sharma, C.~Joe-Wong, and G.~Joshi, ``Fedast: Federated asynchronous simultaneous training,'' in \emph{Proceedings of the Conference on Uncertainty in Artifical Intelligence}, 2024.

\bibitem{bhuyan2022multiProofs}
N.~Bhuyan, S.~Moharir, and G.~Joshi, ``Multi-model federated learning with provable guarantees,'' \emph{arXiv preprint arXiv:2207.04330}, 2022.

\bibitem{bhuyan2022multi}
N.~Bhuyan and S.~Moharir, ``Multi-model federated learning,'' in \emph{2022 14th International Conference on COMmunication Systems \& NETworkS (COMSNETS)}.\hskip 1em plus 0.5em minus 0.4em\relax IEEE, 2022, pp. 779--783.

\bibitem{li2021efficient}
C.~Li, C.~Li, Y.~Zhao, B.~Zhang, and C.~Li, ``An efficient multi-model training algorithm for federated learning,'' in \emph{2021 IEEE Global Communications Conference (GLOBECOM)}.\hskip 1em plus 0.5em minus 0.4em\relax IEEE, 2021, pp. 1--6.

\bibitem{liu2022multi}
J.~Liu, J.~Jia, B.~Ma, C.~Zhou, J.~Zhou, Y.~Zhou, H.~Dai, and D.~Dou, ``Multi-job intelligent scheduling with cross-device federated learning,'' \emph{IEEE Transactions on Parallel and Distributed Systems}, vol.~34, no.~2, pp. 535--551, 2022.

\bibitem{chang2023asynchronous}
Z.-L. Chang, S.~Hosseinalipour, M.~Chiang, and C.~G. Brinton, ``Asynchronous multi-model federated learning over wireless networks: Theory, modeling, and optimization,'' \emph{arXiv preprint arXiv:2305.13503}, 2023.

\bibitem{siew2023fair}
M.~Siew, S.~Arunasalam, Y.~Ruan, Z.~Zhu, L.~Su, S.~Ioannidis, E.~Yeh, and C.~Joe-Wong, ``Fair training of multiple federated learning models on resource constrained network devices,'' in \emph{Proceedings of the 22nd International Conference on Information Processing in Sensor Networks}, 2023, pp. 330--331.

\bibitem{joe2013multiresource}
C.~Joe-Wong, S.~Sen, T.~Lan, and M.~Chiang, ``Multiresource allocation: Fairness--efficiency tradeoffs in a unifying framework,'' \emph{IEEE/ACM Transactions on Networking}, vol.~21, no.~6, pp. 1785--1798, 2013.

\bibitem{altman2008generalized}
E.~Altman, K.~Avrachenkov, and A.~Garnaev, ``Generalized $\alpha$-fair resource allocation in wireless networks,'' in \emph{2008 47th IEEE Conference on Decision and Control}.\hskip 1em plus 0.5em minus 0.4em\relax IEEE, 2008, pp. 2414--2419.

\bibitem{bampis2018fair}
E.~Bampis, B.~Escoffier, and S.~Mladenovic, ``Fair resource allocation over time,'' in \emph{AAMAS 2018-17th International Conference on Autonomous Agents and MultiAgent Systems}.\hskip 1em plus 0.5em minus 0.4em\relax International Foundation for Autonomous Agents and Multiagent Systems, 2018, pp. 766--773.

\bibitem{deng2021fair}
Y.~Deng, F.~Lyu, J.~Ren, Y.-C. Chen, P.~Yang, Y.~Zhou, and Y.~Zhang, ``Fair: Quality-aware federated learning with precise user incentive and model aggregation,'' in \emph{IEEE INFOCOM 2021-IEEE Conference on Computer Communications}.\hskip 1em plus 0.5em minus 0.4em\relax IEEE, 2021, pp. 1--10.

\bibitem{lan2010axiomatic}
T.~Lan, D.~Kao, M.~Chiang, and A.~Sabharwal, \emph{An axiomatic theory of fairness in network resource allocation}.\hskip 1em plus 0.5em minus 0.4em\relax IEEE, 2010.

\bibitem{li2019fair}
T.~Li, M.~Sanjabi, A.~Beirami, and V.~Smith, ``Fair resource allocation in federated learning,'' \emph{arXiv preprint arXiv:1905.10497}, 2019.

\bibitem{multiTaskSurvey}
Y.~Zhang and Q.~Yang, ``A survey on multi-task learning,'' \emph{IEEE Transactions on Knowledge and Data Engineering}, vol.~34, no.~12, pp. 5586--5609, 2022.

\bibitem{ruan2022fedsoft}
Y.~Ruan and C.~Joe-Wong, ``Fedsoft: Soft clustered federated learning with proximal local updating,'' in \emph{Proceedings of the AAAI Conference on Artificial Intelligence}, vol.~36, no.~7, 2022, pp. 8124--8131.

\bibitem{hu2022fedcross}
M.~Hu, P.~Zhou, Z.~Yue, Z.~Ling, Y.~Huang, Y.~Liu, and M.~Chen, ``Fedcross: Towards accurate federated learning via multi-model cross aggregation,'' \emph{arXiv preprint arXiv:2210.08285}, 2022.

\bibitem{xia2020multi}
W.~Xia, T.~Q. Quek, K.~Guo, W.~Wen, H.~H. Yang, and H.~Zhu, ``Multi-armed bandit-based client scheduling for federated learning,'' \emph{IEEE Transactions on Wireless Communications}, vol.~19, no.~11, pp. 7108--7123, 2020.

\bibitem{cho2022towards}
Y.~J. Cho, J.~Wang, and G.~Joshi, ``Towards understanding biased client selection in federated learning,'' in \emph{International Conference on Artificial Intelligence and Statistics}.\hskip 1em plus 0.5em minus 0.4em\relax PMLR, 2022, pp. 10\,351--10\,375.

\bibitem{nishio2019client}
T.~Nishio and R.~Yonetani, ``Client selection for federated learning with heterogeneous resources in mobile edge,'' in \emph{2019 IEEE international conference on communications (ICC)}.\hskip 1em plus 0.5em minus 0.4em\relax IEEE, 2019, pp. 1--7.

\bibitem{nguyen2021toward}
M.~N. Nguyen, N.~H. Tran, Y.~K. Tun, Z.~Han, and C.~S. Hong, ``Toward multiple federated learning services resource sharing in mobile edge networks,'' \emph{IEEE Transactions on Mobile Computing}, vol.~22, no.~1, pp. 541--555, 2021.

\bibitem{poullie2017survey}
P.~Poullie, T.~Bocek, and B.~Stiller, ``A survey of the state-of-the-art in fair multi-resource allocations for data centers,'' \emph{IEEE Transactions on Network and Service Management}, vol.~15, no.~1, pp. 169--183, 2017.

\bibitem{lakhdari2021fairness}
A.~Lakhdari and A.~Bouguettaya, ``Fairness-aware crowdsourcing of iot energy services,'' in \emph{Service-Oriented Computing: 19th International Conference, ICSOC 2021, Virtual Event, November 22--25, 2021, Proceedings 19}.\hskip 1em plus 0.5em minus 0.4em\relax Springer, 2021, pp. 351--367.

\bibitem{basik2018fair}
F.~Bas{\i}k, B.~Gedik, H.~Ferhatosmano{\u{g}}lu, and K.-L. Wu, ``Fair task allocation in crowdsourced delivery,'' \emph{IEEE Transactions on Services Computing}, vol.~14, no.~4, pp. 1040--1053, 2018.

\bibitem{gan2017social}
X.~Gan, Y.~Li, W.~Wang, L.~Fu, and X.~Wang, ``Social crowdsourcing to friends: An incentive mechanism for multi-resource sharing,'' \emph{IEEE Journal on Selected Areas in Communications}, vol.~35, no.~3, pp. 795--808, 2017.

\bibitem{donahue2021models}
K.~Donahue and J.~Kleinberg, ``Models of fairness in federated learning,'' \emph{arXiv preprint arXiv:2112.00818}, 2021.

\bibitem{zhan2021survey}
Y.~Zhan, J.~Zhang, Z.~Hong, L.~Wu, P.~Li, and S.~Guo, ``A survey of incentive mechanism design for federated learning,'' \emph{IEEE Transactions on Emerging Topics in Computing}, vol.~10, no.~2, pp. 1035--1044, 2021.

\bibitem{zhan2020learning}
Y.~Zhan, P.~Li, Z.~Qu, D.~Zeng, and S.~Guo, ``A learning-based incentive mechanism for federated learning,'' \emph{IEEE Internet of Things Journal}, vol.~7, no.~7, pp. 6360--6368, 2020.

\bibitem{weng2021fedserving}
J.~Weng, J.~Weng, H.~Huang, C.~Cai, and C.~Wang, ``Fedserving: A federated prediction serving framework based on incentive mechanism,'' in \emph{IEEE INFOCOM 2021}.\hskip 1em plus 0.5em minus 0.4em\relax IEEE, 2021, pp. 1--10.

\bibitem{wang2022dynamic}
X.~Wang, S.~Zheng, and L.~Duan, ``Dynamic pricing for client recruitment in federated learning,'' \emph{arXiv preprint arXiv:2203.03192}, 2022.

\bibitem{tang2021incentive}
M.~Tang and V.~W. Wong, ``An incentive mechanism for cross-silo federated learning: A public goods perspective,'' in \emph{IEEE Conference on Computer Communications}.\hskip 1em plus 0.5em minus 0.4em\relax IEEE, 2021, pp. 1--10.

\bibitem{MMFLappendix}
\BIBentryALTinterwordspacing
``Tech report.'' [Online]. Available: \url{https://tinyurl.com/MMFLproofs}
\BIBentrySTDinterwordspacing

\bibitem{ruan2021valuable}
Y.~Ruan, X.~Zhang, and C.~Joe-Wong, ``How valuable is your data? optimizing client recruitment in federated learning,'' in \emph{2021 19th International Symposium on Modeling and Optimization in Mobile, Ad hoc, and Wireless Networks (WiOpt)}.\hskip 1em plus 0.5em minus 0.4em\relax IEEE, 2021, pp. 1--8.

\bibitem{singer2014budget}
Y.~Singer, ``Budget feasible mechanism design,'' \emph{ACM SIGecom Exchanges}, vol.~12, no.~2, pp. 24--31, 2014.

\end{thebibliography}


\end{document}